\newif\ifisarxiv
\isarxivtrue

\documentclass{article}

\ifisarxiv
\usepackage{fullpage}
\else
\usepackage[final]{neurips_2020}
\fi
\usepackage[utf8]{inputenc} 
\usepackage[T1]{fontenc}    
\usepackage{hyperref}       
\usepackage{url}            
\usepackage{booktabs}       
\usepackage{amsfonts}       
\usepackage{nicefrac}       
\usepackage{microtype}      
 \usepackage{subfigure}
\usepackage{amsmath,amsthm}
\usepackage{cancel}
\usepackage{cleveref}
\usepackage{mathtools}
\usepackage{enumitem}
\usepackage{todonotes}
\usepackage{xcolor}
\usepackage{xfrac}
\usepackage{caption}
\hypersetup{
	colorlinks,
	linkcolor={red!40!gray},
	citecolor={blue!40!gray},
	urlcolor={blue!70!gray}
}

\newif\ifDRAFT
\DRAFTtrue
\ifDRAFT
\newcommand{\marrow}{\marginpar[\hfill$\longrightarrow$]{$\longleftarrow$}}
\newcommand{\niceremark}[3]
   {\textcolor{red}{\textsc{#1 #2:} \marrow\textsf{#3}}}
\newcommand{\ken}[2][says]{\niceremark{Ken}{#1}{#2}}

\newcommand{\michael}[2][says]{\niceremark{Michael}{#1}{#2}}
\newcommand{\michal}[2][says]{\niceremark{Michal}{#1}{#2}}
\newcommand{\feynman}[2][says]{\niceremark{Feynman}{#1}{#2}}
\else
\newcommand{\ken}[1]{}
\newcommand{\michael}[1]{}
\newcommand{\michal}[1]{}
\newcommand{\feynman}[1]{}
\fi

\def\ee{\mathrm{e}}

\newenvironment{proofof}[2]{\par\vspace{2mm}\noindent\textbf{Proof of {#1} {#2}}\ }{\hfill\BlackBox}

\def\Sigmab{\mathbf{\Sigma}}

\def\S{\mathbf{S}}
\def\T{\mathbf{T}}

\def\g{{\mathbf{g}}}

\def\W{\mathbf W}

\def\H{\mathbf H}
\def\K{\mathbf K}

\ifx\BlackBox\undefined
\newcommand{\BlackBox}{\rule{1.5ex}{1.5ex}}  
\fi
\DeclareMathOperator*{\argmin}{\mathop{\mathrm{argmin}}}

\def\x{\mathbf x}

\def\a{\mathbf a}
\def\b{\mathbf b}

\def\v{\mathbf v}

\def\zero{\mathbf 0}
\def\one{\mathbf 1}

\def\X{\mathbf X}

\def\B{\mathbf B}
\def\A{\mathbf A}
\def\C{\mathbf C}
\def\U{\mathbf U}

\def\D{\mathbf D}
\def\V{\mathbf V}

\def\Z{\mathbf Z}

\def\I{\mathbf I}

\def\A{\mathbf A}
\def\P{\mathbf P}

\def\E{\mathbb E}
\def\R{\mathbb R}

\def\Pr{\mathrm{Pr}}
\def\tr{\mathrm{tr}}

\def\rank{\mathrm{rank}}

\let\origtop\top
\renewcommand\top{{\scriptscriptstyle{\origtop}}} 

\definecolor{silver}{cmyk}{0,0,0,0.3}
\definecolor{yellow}{cmyk}{0,0,0.9,0.0}
\definecolor{reddishyellow}{cmyk}{0,0.22,1.0,0.0}
\definecolor{black}{cmyk}{0,0,0.0,1.0}
\definecolor{darkYellow}{cmyk}{0.2,0.4,1.0,0}
\definecolor{orange}{cmyk}{0.0,0.7,0.9,0}
\definecolor{darkSilver}{cmyk}{0,0,0,0.1}
\definecolor{grey}{cmyk}{0,0,0,0.5}
\definecolor{darkgreen}{cmyk}{0.6,0,0.8,0}

\ifx\proof\undefined
\newenvironment{proof}{\par\noindent{\bf Proof\ }}{\hfill\BlackBox\\[2mm]}
\fi

\ifx\theorem\undefined
\newtheorem{theorem}{Theorem}
\fi

\ifx\example\undefined
\newtheorem{example}{Example}
\fi

\ifx\condition\undefined
\newtheorem{condition}{Condition}
\fi
\ifx\property\undefined
\newtheorem{property}{Property}
\fi

\ifx\lemma\undefined
\newtheorem{lemma}{Lemma}
\fi

\ifx\proposition\undefined
\newtheorem{proposition}{Proposition}
\fi

\ifx\remark\undefined
\newtheorem{remark}{Remark}
\fi

\ifx\corollary\undefined
\newtheorem{corollary}{Corollary}
\fi

\ifx\definition\undefined
\newtheorem{definition}{Definition}
\fi

\ifx\conjecture\undefined
\newtheorem{conjecture}{Conjecture}
\fi

\ifx\axiom\undefined
\newtheorem{axiom}{Axiom}
\fi

\ifx\claim\undefined
\newtheorem{claim}{Claim}
\fi

\ifx\assumption\undefined
\newtheorem{assumption}{Assumption}
\fi

\ifx\condition\undefined
\newtheorem{condition}{Condition}
\fi

\title{Precise expressions for random projections: \\
  Low-rank approximation and randomized Newton\footnotemark}

\ifisarxiv
\author{
          \textbf{Micha{\l } Derezi\'{n}ski} \\
  Department of Statistics\\
  University of California, Berkeley\\
  \texttt{mderezin@berkeley.edu}\\
  \and
  \textbf{Feynman Liang} \\
  Department of Statistics\\
  University of California, Berkeley\\
  \texttt{feynman@berkeley.edu}
  \and
   \textbf{Zhenyu Liao} \\
  ICSI and Department of Statistics\\
  University of California, Berkeley\\
  \texttt{zhenyu.liao@berkeley.edu}
  \and
   \textbf{Michael W. Mahoney}\\
  ICSI and Department of Statistics\\
  University of California, Berkeley\\
  \texttt{mmahoney@stat.berkeley.edu}
}
\else
\author{%
 \textbf{Micha{\l } Derezi\'{n}ski} \\
  Department of Statistics\\
  University of California, Berkeley\\
  \texttt{mderezin@berkeley.edu}\\
  \And
  \textbf{Feynman Liang} \\
  Department of Statistics\\
  University of California, Berkeley\\
  \texttt{feynman@berkeley.edu}
  \And
   \textbf{Zhenyu Liao} \\
  ICSI and Department of Statistics\\
  University of California, Berkeley\\
  \texttt{zhenyu.liao@berkeley.edu}
  \And
   \textbf{Michael W. Mahoney}\\
  ICSI and Department of Statistics\\
  University of California, Berkeley\\
  \texttt{mmahoney@stat.berkeley.edu}
}
\fi

\begin{document}
\maketitle

\begin{abstract}
It is often desirable to reduce the dimensionality of a large dataset
by projecting it onto a low-dimensional subspace.  Matrix sketching
has emerged as a powerful technique for performing such dimensionality
reduction very efficiently.  Even though there is an extensive
literature on the worst-case performance of sketching, existing
guarantees are typically very different from what is observed in
practice.  We exploit recent developments in the spectral analysis of
random matrices to develop novel techniques that provide provably
accurate expressions for the expected value of random projection
matrices obtained via sketching.  These expressions can be used to
characterize the performance of dimensionality reduction in a variety
of common machine learning tasks, ranging from low-rank approximation
to iterative stochastic optimization.  Our results apply to several
popular sketching methods, including Gaussian and Rademacher sketches,
and they enable precise analysis of these methods in terms of spectral
properties of the data.  Empirical results show that the expressions
we derive reflect the practical performance of these sketching
methods, down to lower-order effects and even constant factors. 
\end{abstract}

{\renewcommand{\thefootnote}{\fnsymbol{footnote}}
\footnotetext[1]{This version of the paper includes a correction to
  the assumptions in a technical result, Theorem
  \ref{t:main-tech}. The previous claim relied on a formulation of the
  Hanson-Wright inequality given by \cite[Corollary
  2.8]{zajkowski2020bounds}, which turns out to be false. This was not
  essential for our main results, so none of the other claims are
  affected by this change. The conference version of this paper, i.e.,
\cite{precise-expressions}, does not include the correction, so we
recommend to cite this arXiv version when referencing Theorem~\ref{t:main-tech}. }}

\section{Introduction}
Many settings in modern machine learning, optimization and scientific
computing require us to work with data matrices that are so large that
some form of dimensionality reduction is a necessary component of the
process. One of the most popular families of methods for dimensionality reduction,
coming from the literature on Randomized Numerical Linear Algebra
(RandNLA), consists of data-oblivious sketches \cite{Mah-mat-rev_JRNL,tropp2011structure,woodruff2014sketching}. 
Consider a large $m\times n$
matrix $\A$. A \emph{data-oblivious sketch} of size $k$ is the matrix $\S\A$,
where $\S$ is a $k\times m$ random matrix such that
$\E[\frac1k\S^\top\S]=\I$, whose distribution does not
depend on $\A$. This sketch reduces the first dimension of $\A$ from
$m$ to a much smaller $k$ (we assume without loss of generality that
$k \ll n \le m$), and an analogous procedure can be defined for
reducing the second 
dimension as well. This approximate representation of $\A$ is central
to many algorithms in areas such as linear regression, low-rank approximation, kernel methods,
and iterative second-order optimization. While there is a long line of research
aimed at bounding the worst-case approximation error of such
representations, these bounds are often too loose to reflect accurately
the practical performance of these methods. In this paper, we develop new theory which
enables more precise analysis of the accuracy of sketched data
representations.

A common way to measure the accuracy of the sketch $\S\A$ is by
considering the $k$-dimensional subspace spanned by its rows. The goal
of the sketch is to choose a subspace that best aligns with the
distribution of all of the $m$ rows of $\A$ in $\R^n$. Intuitively, our goal
is to minimize the (norm of the) residual when projecting a vector $\a\in\R^n$ onto that
subspace, i.e., $\a - \P\a=(\I-\P)\a$, where
$\P = (\S\A)^\dagger\S\A$ is the orthogonal projection matrix onto the subspace
spanned by the rows of $\S\A$ (and $(\cdot)^\dagger$ denotes the
Moore-Penrose pseudoinverse). For this reason, the quantity that has appeared 
ubiquitously in the error analysis of RandNLA sketching is what we call the residual
projection matrix:
\begin{align*}
  \textbf{(residual projection matrix)}\quad \P_{\!\perp}\ :=\ \I - \P\ =\ \I
  - (\S\A)^\dagger\S\A.
\end{align*}
Since $\P_{\!\perp}$ is random, the average performance of the sketch
can often be characterized by its expectation, $\E[\P_{\!\perp}]$.
For example, the low-rank approximation error of the
sketch can be expressed as
$\E[\|\A - \A\P\|_F^2]= \tr\,\A^\top\A\,\E[\P_{\!\perp}]$, where
$\|\cdot\|_F$ denotes the Frobenius norm.  A similar formula follows for
the trace norm error of a sketched Nystr\"om approximation
\cite{Williams01Nystrom,revisiting-nystrom}.
Among others, this approximation error
appears in the analysis of sketched kernel 
ridge regression \cite{fanuel2020diversity} and Gaussian process
regression \cite{sparse-variational-gp}. Furthermore, a variety of
iterative algorithms, such as randomized second-order methods for
convex optimization \cite{Qu2015Feb,Qu2016,Gower2019,jacsketch}
and linear system solvers based on 
the generalized Kaczmarz method \cite{generalized-kaczmarz},
have convergence guarantees which depend on the extreme eigenvalues of
$\E[\P_{\!\perp}]$. Finally, a generalized form of the expected
residual projection has been recently used to model the implicit
regularization of the interpolating solutions in over-parameterized
linear models \cite{surrogate-design,BLLT19_TR}.

\subsection{Main result}
Despite its prevalence in the literature, the 
expected residual projection is not well understood, even in such
simple cases as when $\S$ is a Gaussian sketch (i.e., with
i.i.d.~standard normal entries). We address this by providing a
surrogate expression, i.e., a simple analytically
tractable approximation, for this matrix quantity:
\begin{align}
  \E[\P_{\!\perp}]\ \overset\epsilon\simeq \ \bar\P_{\!\perp}:=(\gamma\A^\top\A +
  \I)^{-1},\quad\text{with \ $\gamma>0$ \ s.t. \ }\tr\,\bar\P_{\!\perp} = n-k.\label{eq:surrogate}
\end{align}
Here, $\overset{\epsilon}{\simeq}$ means that while the surrogate expression is
not exact, it approximates the true quantity up to some $\epsilon$
accuracy. Our main result provides a rigorous approximation guarantee
for this surrogate expression with respect to a
range of sketching matrices $\S$, including the standard Gaussian and
Rademacher sketches. We state the result using the
positive semi-definite ordering denoted by $\preceq$. 
\begin{theorem}\label{t:main}
Let $\S$ be a sketch of size $k$ with i.i.d.~mean-zero sub-gaussian entries and let
  $r=\|\A\|_F^2/\|\A\|^2$ be the stable rank of $\A$.  If
  we let $\rho = r/k$ be a fixed constant larger than $1$, then
  \begin{align*}
    (1-\epsilon)\,\bar\P_{\!\perp}\preceq\E[\P_{\!\perp}]\preceq
    (1+\epsilon)\,\bar\P_{\!\perp}\quad\text{for}\quad \epsilon =
    O(\tfrac1{\sqrt r}).
  \end{align*}
\end{theorem}
In other words, when the sketch size $k$ is smaller than the stable rank $r$
of $\A$, then the discrepancy between our surrogate expression
$\bar\P_{\!\perp}$ and $\E[\P_{\!\perp}]$ is of the order
$1/\sqrt r$, where the big-O notation hides only the dependence
on $\rho$ and on the sub-gaussian constant (see Theorem \ref{t:main-tech} for
more details). Our proof of Theorem \ref{t:main} is inspired by the techniques from
random matrix theory which have been used to analyze the asymptotic
spectral distribution of large random matrices by focusing on the
associated matrix resolvents and Stieltjes transforms
\cite{hachem2007deterministic,bai2010spectral}. However, our analysis
is novel in several respects:
\begin{enumerate}
  \item The residual projection matrix can be obtained
    from the appropriately scaled resolvent matrix $z(\A^\top\S^\top\S\A+z\I)^{-1}$ by
    taking $z\rightarrow 0$. Prior work (e.g.,
      \cite{HMRT19_TR}) combined this with an exchange-of-limits argument to
      analyze the asymptotic behavior of the residual projection. 
      This approach, however, does not allow for a precise control in
      finite-dimensional problems. We are able to provide a more
      fine-grained, non-asymptotic analysis by working directly with
      the residual projection itself, instead of the resolvent.
  \item We require no assumptions on the largest and smallest singular
    value of $\A$. Instead, we derive our bounds in terms of the
    stable rank of $\A$ (as opposed to its actual rank), which
    implicitly compensates for ill-conditioned data matrices.
  \item We obtain upper/lower bounds for $\E[\P_{\!\perp}]$ in terms
    of the positive semi-definite ordering~$\preceq$, which can be
    directly converted to guarantees for the precise
    expressions of expected low-rank approximation error derived in
    the following section. 
  \end{enumerate}

It is worth mentioning that the proposed analysis is significantly
different from the sketching literature based on subspace
embeddings (e.g.,
\cite{sarlos-sketching,cw-sparse,nn-sparse,projection-cost-preserving,optimal-matrix-product}),
in the sense that here our object of interest is not to obtain
a worst-case approximation with high probability, but rather, our
analysis provides \emph{precise} characterization on the
\emph{expected} residual projection matrix that goes \emph{beyond
  worst-case bounds}. From an application perspective, the subspace
embedding property is neither sufficient nor necessary for many
numerical implementations of sketching \cite{blendenpik,lsrn}, 
or statistical results \cite{GarveshMahoney_JMLR,dobriban2019asymptotics,yang2020reduce}, as well as in the context of iterative optimization and implicit regularization (see Sections~\ref{s:newton}~and~\ref{s:implicit} below), which are discussed in detail as concrete applications of the proposed analysis.

\subsection{Low-rank approximation}
\label{s:low-rank}
We next provide some immediate corollaries of Theorem~\ref{t:main},
where we use $x\overset\epsilon\simeq y$ to denote a 
multiplicative approximation $|x-y|\leq \epsilon y$.
Note that our analysis is new even for the classical Gaussian sketch where
the entries of $\S$ are i.i.d.~standard normal. However the results
apply more broadly, including a standard class of data-base friendly
Rademacher sketches where each entry $s_{ij}$ is a $\pm1$ Rademacher
random variable \cite{achlioptas2003database}. We start by analyzing the Frobenius
norm error $\|\A-\A\P\|_F^2=\tr\,\A^\top\A\,\P_{\!\perp}$ of
sketched low-rank approximations. Note that by the definition of $\gamma$ in
\eqref{eq:surrogate}, we have $k =
\tr\,(\I-\bar\P_{\!\perp})=\tr\,\gamma\A^\top\A(\gamma\A^\top\A+\I)^{-1}$,
so the surrogate expression we obtain for the expected error is
remarkably simple.
\begin{corollary} 
  \label{c:low-rank}
Let $\sigma_i$ be the singular values of $\A$. Under the assumptions of Theorem \ref{t:main}, we have:\vspace{-1mm}
  \begin{align*}
    \E\big[\|\A-\A\P\|_F^2\big] \ \overset\epsilon\simeq\ 
    k/\gamma
\quad \text{for \ $\gamma>0$ \ s.t. \ } \sum_{i}\frac{\gamma\sigma_i^2}{\gamma\sigma_i^2+1} = k.
  \end{align*}
\end{corollary}\vspace{-4mm}
\begin{remark}
  The parameter $\gamma=\gamma(k)$ increases at least linearly
  as a function of $k$, which is why the expected error will always
  decrease with increasing $k$. For example, when the singular values
  of $\A$ exhibit exponential decay, i.e., $\sigma_i^2=C\cdot\alpha^{i-1}$ for $\alpha\in(0,1)$, then the error also decreases
  exponentially, at the rate of $k/(\alpha^{-k}-1)$.
We discuss this further in Section \ref{s:explicit}, giving explicit
formulas for the error as a function of $k$ under both exponential and polynomial spectral
decay profiles.
\end{remark}
The above result is important for many RandNLA methods, and it is also relevant in the context of kernel methods,
where the data is represented via a positive semi-definite $m\times m$
kernel matrix $\K$ which corresponds to the matrix of dot-products of
the data vectors in some reproducible kernel Hilbert space. In this
context, sketching can be applied directly to the matrix $\K$ via an
extended variant of the
Nystr\"om method \cite{revisiting-nystrom}. A Nystr\"om approximation
constructed from a sketching matrix $\S$ is defined as $\tilde\K =
\C^\top\W^\dagger\C$, where $\C=\S\K$ and $\W=\S\K\S^\top$, and it is
applicable to a variety of settings, including Gaussian Process
regression, kernel machines and Independent
Component Analysis \cite{sparse-variational-gp,Williams01Nystrom,Bach2003}. By setting
$\A=\K^{\frac12}$, it is easy to see \cite{nystrom-multiple-descent} that
the trace norm error $\|\K-\tilde\K\|_*$ is identical to the squared
Frobenius norm error of the low-rank 
sketch $\S\A$, so Corollary \ref{c:low-rank} implies that
\begin{align}
  \E\big[\|\K-\tilde\K\|_*\big] \ \overset\epsilon\simeq\  k/\gamma \quad \text{for \ $\gamma>0$ \ s.t. \ } \sum_{i}\frac{\gamma\lambda_i}{\gamma\lambda_i+1} = k,\label{eq:nystrom}
\end{align}
with any sub-gaussian sketch, where $\lambda_i$ denote the
eigenvalues of $\K$. Our error analysis 
given in Section \ref{s:explicit} is particularly relevant here, since
commonly used kernels such as the Radial 
Basis Function (RBF) or the Mat\'ern kernel induce a well-understood eigenvalue
decay \cite{Santa97Gaussianregression,RasmussenWilliams06}.

Metrics other than the aforementioned Frobenius norm error,
such as the spectral norm error \cite{tropp2011structure}, are also of significant
interest in the low-rank approximation literature. We leave these
directions for future investigation.

\subsection{Randomized iterative optimization}
\label{s:newton}
We next turn to a class of iterative methods which take
advantage of sketching to reduce the per iteration cost of
optimization. These methods have been developed in a variety of
settings, from solving linear systems to convex optimization and
empirical risk minimization, and in many cases
the residual projection matrix appears as a black box quantity whose spectral
properties determine the convergence behavior of the algorithms
\cite{generalized-kaczmarz}. With
our new results, we can precisely characterize not only the rate of
convergence, but also, in some cases, the complete evolution of
the parameter vector, for the following algorithms:
\begin{enumerate}
  \item \emph{Generalized Kaczmarz method}
    \cite{generalized-kaczmarz} for approximately solving a linear system $\A\x=\b$;
  \item \emph{Randomized Subspace Newton} \cite{Gower2019}, a second order
    method, where we sketch the Hessian matrix.
  \item \emph{Jacobian Sketching} \cite{jacsketch}, a
    class of first order methods which use additional information via a
    weight matrix $\W$ that is sketched at every iteration.
\end{enumerate}
 We believe that
extensions of our techniques will apply to other algorithms,
such as that of \cite{lacotte2019high}. 

We next 
give a result in the context of linear systems
for the generalized Kaczmarz method \cite{generalized-kaczmarz}, but a similar
convergence analysis is given for the methods of
\cite{Gower2019,jacsketch} in Appendix~\ref{a:newton}.
\begin{corollary}\label{c:kaczmarz}
Let $\x^*$ be the unique solution of $\A\x^*=\b$ and consider
  the iterative algorithm:
  \begin{align*}
    \x^{t+1} = \argmin_\x\|\x-\x^t\|^2\quad\textnormal{subject to}\quad\S\A\x=\S\b.
  \end{align*}
Under the assumptions of Theorem \ref{t:main}, with $\gamma$ defined in
\eqref{eq:surrogate} and $r=\|\A\|_F^2/\|\A||^2$, we have:
  \begin{align*}
    \E\big[\x^{t+1}-\x^*\big] \overset\epsilon\simeq 
    (\gamma\A^\top\A+\I)^{-1}\,\E\big[\x^t-\x^*\big]
    \quad\text{for}\quad\epsilon=O(\tfrac1{\sqrt r}). 
  \end{align*}
\end{corollary}
The corollary follows from Theorem \ref{t:main} combined with Theorem
4.1 in \cite{generalized-kaczmarz}. 
Note that when $\A^\top\A$ is positive definite then
$(\gamma\A^\top\A+\I)^{-1}\prec\I$, so the algorithm will
converge from any starting point, and the worst-case convergence rate of the above method can be
obtained by evaluating the largest eigenvalue of
$(\gamma\A^\top\A+\I)^{-1}$. However the result itself is much
stronger, in that it can be used to describe the (expected) trajectory of the
iterates for any starting point $\x^0$. Moreover, when the spectral decay
profile of $\A$ is known, then the explicit expressions for $\gamma$
as a function of $k$ derived in Section \ref{s:explicit} can be used
to characterize the convergence properties of generalized Kaczmarz as
well as other methods discussed above.

\subsection{Implicit regularization}\label{s:implicit}

Setting $\x^t=\zero$, we can view one step of the iterative method in Corollary~\ref{c:kaczmarz} as
  finding a minimum norm interpolating solution of an under-determined linear
  system $(\S\A,\S\b)$. Recent interest in the generalization
capacity of over-parameterized machine learning models has
  motivated extensive research on the statistical properties of such interpolating
  solutions, e.g., \cite{BLLT19_TR,HMRT19_TR,surrogate-design}. In this
  context, Theorem \ref{t:main} provides new evidence for the implicit
  regularization conjecture posed by \cite{surrogate-design} (see
  their Theorem 2 and associated discussion), with the amount of
  regularization equal $\frac1\gamma$, where $\gamma$ is implicitly
  defined in \eqref{eq:surrogate}:
  \begin{align*}
  \underbrace{\E\Big[\argmin_\x\|\x\|^2\ \ \textnormal{s.t.}\ \
    \S\A\x=\S\b\Big] - \x^*}_{\text{Bias of sketched minimum norm solution}}
    \  \ \overset\epsilon\simeq\ \
    \underbrace{\argmin_\x\Big\{\|\A\x-\b\|^2+\tfrac1\gamma\|\x\|^2\Big\}
    - \x^*}_{\text{Bias of $l_2$-regularized solution}}.
  \end{align*}
While implicit regularization has received attention recently in the context of SGD algorithms for overparameterized machine learning models, it was originally discussed in the context of approximation algorithms more generally~\cite{Mah12}.
Recent work has made precise this notion in the context of RandNLA \cite{surrogate-design}, and our results here can be viewed in terms of implicit regularization of scalable RandNLA methods.

\vspace{-1mm}
\subsection{Related work}

A significant body of research has been dedicated to understanding the
guarantees for low-rank approximation via sketching,
particularly in the context of 
RandNLA \cite{DM16_CACM,RandNLA_PCMIchapter_chapter}. This line of work includes
i.i.d.~row sampling methods \cite{BoutsidisMD08,ridge-leverage-scores} which 
preserve the structure of the data, and data-oblivious methods such as
Gaussian and Rademacher sketches \cite{Mah-mat-rev_JRNL,tropp2011structure,woodruff2014sketching}. However,
all of these results focus on worst-case 
upper bounds on the approximation error. One exception is a recent
line of works on non-i.i.d.~row sampling with Determinantal Point Processes
(DPP, \cite{dpps-in-randnla}). In this case, exact analysis of the low-rank approximation
error \cite{nystrom-multiple-descent}, as well as precise convergence analysis
of stochastic second order methods \cite{randomized-newton}, have been
obtained. Remarkably, the expressions they obtain are analogous to
\eqref{eq:surrogate}, despite using completely different techniques. However, their analysis
is limited only to DPP-based sketches, which are considerably more
expensive to construct and thus much less widely used. The connection
between DPPs and Gaussian sketches was recently explored by 
\cite{surrogate-design} in the context of analyzing the implicit
regularization effect of choosing a minimum norm solution in
under-determined linear regression. They conjectured that the
expectation formulas obtained for DPPs are a good proxy for the
corresponding quantities obtained under a Gaussian
distribution. Similar observations were made by
\cite{debiasing-second-order} in the context of sketching for
regularized least squares and second order optimization. While
both of these works only provide empirical evidence for this
particular claim, our Theorem \ref{t:main} can be 
viewed as the first theoretical non-asymptotic justification of that
conjecture.

The effectiveness of sketching has also been extensively studied in
the context of second order optimization. These methods differ
depending on how the sketch is applied to the Hessian matrix, and
whether or not it is applied to the gradient as well. The class of
methods discussed in Section \ref{s:newton}, including Randomized
Subspace Newton and the Generalized Kaczmarz method, relies on
projecting the Hessian downto a low-dimensional subspace, which makes our
results directly applicable. A related family of methods uses
the so-called Iterative Hessian Sketch (IHS) approach \cite{pilanci2016iterative,lacotte2019faster}. The
similarities between IHS and the Subspace Newton-type methods (see
\cite{Qu2015Feb} for a comparison) suggest that our techniques could be
extended to provide precise convergence guarantees also to the IHS.
Finally, yet another family of Hessian sketching methods has been studied by
\cite{roosta2019sub,sketched-ridge-regression,XRM17_theory_TR,YXRM18_TR,fred_newtonMR_TR,distributed-newton,determinantal-averaging}.
These methods  preserve the  rank of the 
 Hessian, and so their convergence
guarantees do not rely on the residual projection.

\section{Precise analysis of the residual projection}

  In this section, we give a detailed statement of our main technical
result, along with a sketch of the proof. First, recall the definition of
sub-gaussian random variables. We say that $x$ is a
$K$-sub-gaussian random variable if its sub-gaussian Orlicz norm is
bounded by $K$, i.e., $\| x
\|_{\psi_2} \le K$, where $\| x \|_{\psi_2} := \inf\{ t > 0:~\E[
\exp(x^2/t^2) ] \le 2 \}$.

For the sake of generality, we state the main result in a slightly different form
than Theorem~\ref{t:main}, which is potentially of independent
interest to random matrix theory and high-dimensional
statistics. Namely, we replace the $m\times n$ matrix 
$\A$ with a  positive semi-definite $n\times n$ matrix
$\Sigmab^{\frac12}$. Furthermore, instead of a sketch $\S$ with
i.i.d.~sub-gaussian entries, we use a random matrix $\Z$ with
i.i.d.~isotropic rows, so that the random matrix
$\X=\Z\Sigmab^{\frac12}$ (which replaced the sketch $\S\A$) represents
random row samples from an $n$-variate distribution
with covariance $\Sigmab$. We do not require the rows of $\Z$ to have
independent entries, but rather, that they satisfy a sub-gaussian concentration
property known as the Hanson-Wright inequality.
\begin{definition}\label{d:hanson-wright}
A random $n$-dimensional vector $\x$ satisfies the Hanson-Wright
inequality with constant $K$ if:
\vspace{-2mm}
\begin{align*}
    \Pr\big\{|\x^\top\B\x-\tr(\B)|\geq t\big\}\leq
  2\exp\bigg(-\min\Big\{\frac{t^2}{K^4\|\B\|_F^2},\frac{t}{K^2\|\B\|}\Big\}\bigg)
  \qquad\text{for any $n\times n$ matrix $\B$.}
\end{align*}
\end{definition}
Any isotropic random vector with independent mean zero $K$-sub-gaussian entries
satisfies the Hanson-Wright inequality with constant $O(K)$
\cite{rudelson2013hanson}, but the inequality can also be satisfied by
vectors with dependent entries, which is why it is a strictly weaker condition.
In Section \ref{s:reduction} we show how to 
convert this more general setup from $\Z$ and $\Sigmab^{\frac12}$ back to the
statement with $\S$ and $\A$ given in Theorem~\ref{t:main}.

\begin{theorem}\label{t:main-tech}
Let $\P_\perp = \I - \X^\dagger \X$ for $\X = \Z\Sigmab^{\frac12}$,
where $\Z \in \mathbb R^{k\times n}$ has i.i.d.~rows with zero mean
and identity covariance that satisfy the Hanson-Wright inequality with
constant $K$, and
$\Sigmab$ is an $n\times n$ positive semi-definite matrix. Define:
\ifisarxiv\vspace{-1.25mm}\fi
\begin{align*}
\bar\P_\perp= (\gamma\Sigmab + \I)^{-1},\quad\text{such that}\quad\tr\,\bar\P_\perp=n-k.
\end{align*}
\ifisarxiv\vspace{-4.75mm}

\noindent
\fi
Let $r =\tr(\Sigmab)/\|\Sigmab\|$
be the stable rank of
$\Sigmab^{\frac12}$ and fix $\rho=r/k > 1$. There exists a constant
$C_{\rho}>0$, depending only on $\rho$ and $K$, such that if $r\geq
C_\rho$, then
\ifisarxiv\vspace{-1.5mm}\fi
\begin{align}
\Big(1-\frac {C_\rho}{\sqrt
  r}\Big)\cdot\bar\P_\perp\preceq\E[\P_\perp]\preceq
  \Big(1+\frac {C_\rho}{\sqrt r}\Big)\cdot\bar\P_\perp.
\end{align}
\end{theorem}

We first provide the following informal derivation of the expression for $\bar \P_\perp$
given in Theorem~\ref{t:main-tech}. 
Let us use $\P$ to denote the matrix $\X^\dagger \X=\I-\P_{\perp}$. Using a 
rank-one update formula for the Moore-Penrose pseudoinverse (see
Lemma~\ref{l:rank-one-update} in the appendix) we have
\[\I-\E[\P_{\perp}]=\E[\P] = \E\big[(\X^\top \X)^\dagger \X^\top \X\big]= \sum_{i=1}^k
  \E[(\X^\top \X)^\dagger \x_i \x_i^\top]  = k\, \E\! \left[\frac{ (\I
      - \P_{-k}) \x_k \x_k^\top }{ \x_k^\top (\I - \P_{-k}) \x_k }
  \right],\]
where we use $\x_i^\top$ to denote the $i$-th row of $\X$, and
$\P_{-k} = \X_{-k}^\dagger \X_{-k} $, where $\X_{-i} $ is the matrix
$\X$ without its $i$-th row. Thanks to the 
Hanson-Wright inequality, the quadratic form $\x_k^\top (\I -
\P_{-k}) \x_k$ in the denominator concentrates around
its expectation (with respect to $\x_k$), i.e., $\tr\Sigmab(\I-\P_{-k})$,
where we use $\E[\x_k \x_k^\top] = \Sigmab$. Further note that, with
$\P_{-k} \simeq \P$ for large $k$ and $\frac1k \tr \Sigmab (\I -
\P_{-k}) \simeq \frac1k \tr \Sigmab \E[\P_\perp]$ from a concentration
argument, we conclude~that
\begin{align*}
\I - \E[\P_\perp] \simeq \frac{ k \E[\P_\perp] \Sigmab }{
  \tr \Sigmab \E[\P_\perp] }\qquad\Longrightarrow\qquad \E[\P_\perp] \simeq  \Big( \frac{ k\Sigmab}{ \tr \Sigmab \E[\P_\perp] } + \I \Big)^{-1},
\end{align*}
  and thus $\E[\P_\perp] \simeq \bar \P_\perp$ for $\bar \P_\perp =
(\gamma \Sigmab + \I)^{-1}$ and $\gamma^{-1} = \frac1k \tr \Sigmab
\bar \P_\perp$. This leads to the (implicit) expression for $\bar
\P_\perp$ and $ \gamma$ given in Theorem \ref{t:main-tech}.  

\medskip

\subsection{Proof sketch of Theorem~\ref{t:main-tech}}
\label{subsec:proof-sketch}

To make the above intuition rigorous, we next present a proof sketch
for Theorem~\ref{t:main-tech}, with the detailed proof deferred to
Appendix~\ref{sec:proof-of-theo-main-tech}.  The proof can be divided
into the following three steps.

\paragraph{Step 1.} First note that, to obtain the lower and upper bound for $\E[\P_\perp]$ in the sense of symmetric matrix as in Theorem~\ref{t:main-tech}, it suffices to bound the spectral
norm $\| \I - \E[\P_\perp] \bar \P_\perp^{-1} \| \le \frac{C_\rho}{\sqrt r}$, so that, with $\frac{\rho-1}{\rho}\I \preceq \bar\P_\perp \preceq \I$ for $\rho = r/k > 1$ from the definition of $\bar \P_\perp$, we have
\[
  \|\I-\bar\P_\perp^{-\frac12}\E[\P_\perp]\bar\P_\perp^{-\frac12}\| =
    \|\bar\P_\perp^{-\frac12}(\I-\E[\P_\perp]\bar\P^{-1})\bar\P_\perp^{\frac12}\|\leq
  \frac{C_\rho}{\sqrt r} \sqrt{\frac{\rho}{\rho-1}} =: \epsilon.
\]
This means that all eigenvalues 
  of the p.s.d.\@ matrix $\bar\P_\perp^{-\frac12}\E[\P_\perp]\bar\P_\perp^{-\frac12}$ lie
  in the interval $[1-\epsilon,1+\epsilon]$, so $(1-\epsilon)\I\preceq
    \bar\P_\perp^{-\frac12}\E[\P_\perp]\bar\P_\perp^{-\frac12}\preceq (1+\epsilon)\I.$ Multiplying by $\bar\P_{\perp}^{\frac12}$ from both sides, we obtain the
  desired bound.
 \paragraph{Step 2.} Then, we carefully design an event $E$ that (i)
 is provable to occur with high probability and (ii) ensures that the
 denominators in the following decomposition are bounded away from
 zero:
  \begin{align*}
  \I-\E[\P_\perp]\bar\P_\perp^{-1}\!
    &= \E[\P] - \gamma\E[\P_\perp]\Sigmab = \E[\P \cdot \one_E] +
    \E[\P  \cdot \one_{\neg E}]
    -\gamma\E[\P_\perp]\Sigmab\\
  &=\gamma\underbrace{\E\bigg[(\bar s-\hat
    s)\,\frac{(\I -\P_{-k})\x_k\x_k^\top}{\x_k^\top(\I -\P_{-k})\x_k} \cdot \one_E\bigg]}_{\T_1}
    - \gamma \underbrace{\E[ (\I - \P_{-k}) \x_k\x_k^\top \cdot
    \one_{\neg E}]}_{\T_2}
         + \gamma \underbrace{\E[\P-\P_{-k}]\Sigmab}_{\T_3} + \underbrace{\E[\P \cdot \one_{\neg E}]}_{\T_4},
\end{align*}
where we let $\hat s = \x_k^\top(\I - \P_{-k})\x_k$ and $\bar s=k/\gamma$. 
\paragraph{Step 3.}
It then remains to bound the spectral norms of $ \T_1, \T_2, \T_3
,\T_4$ respectively to reach the conclusion. More precisely, the terms
$\| \T_2 \| $ and $\| \T_4 \|$ are proportional to $\Pr (\neg E)$,
while the term $\| \T_3 \|$ can be bounded using the rank-one update
formula for the pseudoinverse (Lemma~\ref{l:rank-one-update} in the
appendix). The remaining term $\| \T_1 \|$ is more subtle and can be
bounded with a careful application of  the Hanson-Wright inequality.
This allows for a bound on the operator norm $\| \I -
\E[\P_\perp ] \bar \P_\perp^{-1}\| $ and hence the conclusion.

\subsection{Proof of Theorem \ref{t:main}}\label{s:reduction}
We now discuss how Theorem \ref{t:main} can be obtained from Theorem
\ref{t:main-tech}. The crucial difference between the statements is
that in Theorem \ref{t:main} we let $\A$ be an arbitrary rectangular
matrix, whereas in Theorem \ref{t:main-tech} we instead use a square,
symmetric and positive semi-definite matrix $\Sigmab$. To convert between the
two notations, consider the SVD decomposition
$\A=\U\D\V^\top$ of $\A \in \R^{m \times n}$ (recall that we assume $m
\ge n$), where $\U\in\R^{m\times n}$ and $\V\in\R^{n\times 
  n}$ have orthonormal columns and $\D$ is a diagonal matrix. Now, let
$\Z = \S\U$, $\Sigmab=\D^2$ and  
$\X=\Z\Sigmab^{\frac12}=\S\U\D$. Using the fact that
$\V^\top\V=\V\V^\top=\I$, it follows that:
\begin{align*}
  \I - (\S\A)^\dagger\S\A = \V(\I-\X^\dagger\X)\V^\top\quad\text{and}\quad
  (\gamma\A^\top\A+\I)^{-1} = \V(\gamma\Sigmab + \I)^{-1}\V^\top.
\end{align*}
Since the rows of $\S$ consist of mean zero unit variance
i.i.d. sub-gaussian entries, they satisfy the Hanson-Wright inequality
for any matrix \cite[Theorem 1]{rudelson2013hanson}, including matrices of the
form $\U\B\U^\top$. Thus, the rows of $\Z=\S\U$ also satisfy
Hanson-Wright with the same constant (even though their entries are
not necessarily independent).
Moreover,
using the fact that $\B\preceq\C$ implies
$\V\B\V^\top\preceq\V\C\V^\top$ for any p.s.d.~matrices $\B$ and $\C$,
Theorem \ref{t:main} follows as a corollary of Theorem \ref{t:main-tech}.

\section{Explicit formulas under known spectral decay}
\label{s:explicit}
The expression we give for the expected residual projection,
$\E[\P_{\perp}]\simeq (\gamma\A^\top\A+\I)^{-1}$, is implicit
in that it depends on the parameter $\gamma$ which is the solution of
the following equation:
\begin{align}
  \sum_{i\geq 1}\frac{\gamma\sigma_i^2}{\gamma\sigma_i^2+1} =
  k,\qquad\text{where $\sigma_i$ are the singular values of $\A$.}\label{eq:equation}
\end{align}
In general, it is impossible to solve this equation analytically, i.e., to write
$\gamma$ as an explicit formula of $n$, $k$ and
the singular values of $\A$. However, we show that when the 
singular values exhibit a known rate of decay, then it is
possible to obtain explicit formulas for $\gamma$. In particular, this
allows us to provide precise and easily interpretable rates of decay for the low-rank
approximation error of a sub-gaussian sketch.

Matrices that have known spectral decay, most commonly with either
exponential or polynomial rate, arise in many machine
learning problems \cite{randomized-newton}. Such behavior can be naturally
occurring in data, or it can be induced by feature expansion using, say, the
RBF kernel (for exponential decay) \cite{Santa97Gaussianregression} or the
Mat\'ern kernel (for polynomial decay) \cite{RasmussenWilliams06}. Understanding these
two classes of decay plays an important role in distinguishing the
properties of light-tailed and heavy-tailed data distributions.
Note that in the kernel setting we may often represent our data via
the $m\times m$ kernel matrix $\K$, instead of the $m\times n$ data
matrix $\A$, and study the sketched Nystr\"om method
\cite{revisiting-nystrom} for low-rank approximation. To handle 
the kernel setting in our analysis, it suffices to replace the squared singular
values $\sigma_i^2$ of $\A$ with the eigenvalues of $\K$.
\begin{figure}[t]
\centering
\subfigure[%
Singular values are given by $\sigma_i^2=C\cdot \alpha^{i-1}$.]{%
    \includegraphics[width=0.48\textwidth]{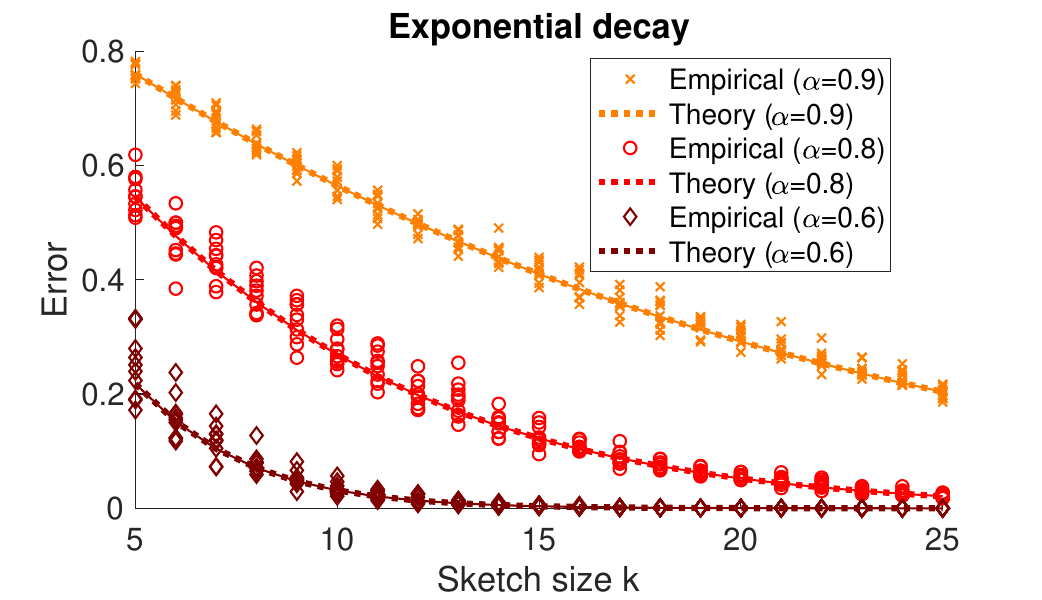}
  }
\hfill
\subfigure[%
Singular values are given by $\sigma_i^2=C\cdot i^{-\beta}$.
]{%
      \includegraphics[width=0.48\textwidth]{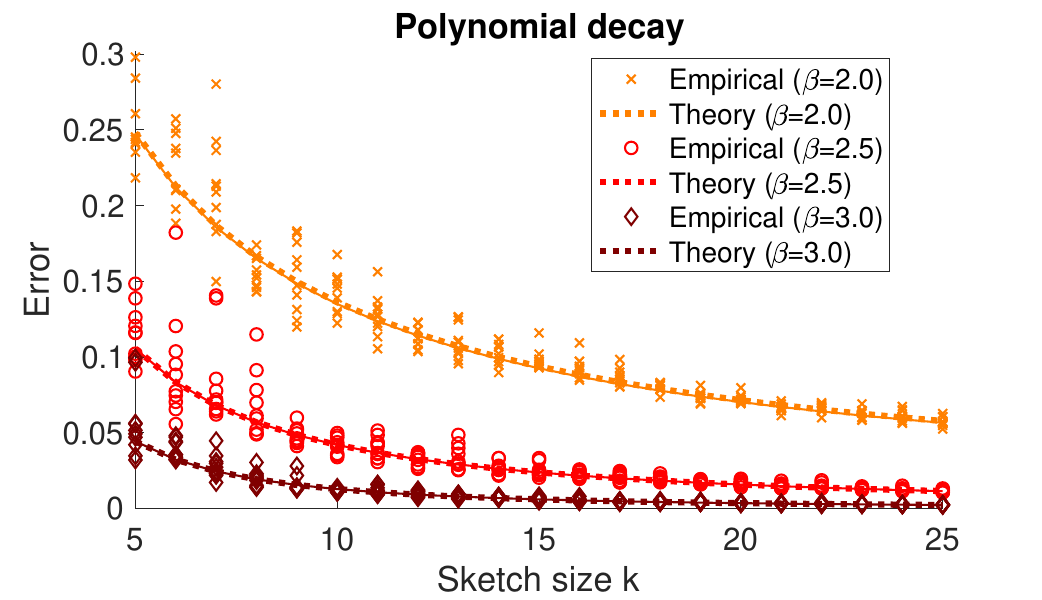}
    }
\caption{
Theoretical predictions of low-rank approximation error of a Gaussian
sketch under known spectral decays, compared to the empirical results.
The constant $C$ is scaled so that $\|\A\|_F^2=1$ and we let
$n=m=1000$. For the theory, we plot the explicit formulas
\eqref{eq:explicit-exp} and \eqref{eq:explicit-poly} (dashed lines),
as well as the implicit expression from Corollary \ref{c:low-rank}
(thin solid lines) obtained by numerically solving
\eqref{eq:equation}. Observe that the explicit and implicit
predictions are nearly (but not exactly) identical.
}
\label{f:explicit}
\end{figure}

\subsection{Exponential spectral decay}
Suppose that the squared singular values of $\A$ exhibit exponential
decay, i.e. $\sigma_i^2 = C\cdot \alpha^{i-1}$, where $C$ is a constant
and $\alpha\in(0,1)$. For simplicity of presentation, we will let
$m,n\rightarrow\infty$. Under this spectral decay, we can approximate
the sum in \eqref{eq:equation} by the analytically computable integral
$\int_y^\infty\!\frac1{1+(C\gamma)^{-1}\alpha^{-x}} dx$,
obtaining $\gamma\approx(\alpha^{-k}-1)\sqrt\alpha/C$. Applying this to the
formula from Corollary~\ref{c:low-rank}, we can express the low-rank
approximation error  
for a sketch of size $k$ as follows:
\begin{align}
  \E\big[\|\A-\A\P\|_F^2\big] \approx
\frac C{\sqrt\alpha}\cdot
  \frac{k}{\alpha^{-k}-1},\quad\text{when}\quad\sigma_i^2=C\cdot\alpha^{i-1}\
  \ \text{for all $i$.}\label{eq:explicit-exp}
\end{align}
In Figure \ref{f:explicit}a, we plot the above formula against the
numerically obtained implicit expression from
Corollary~\ref{c:low-rank}, as well as empirical results for a Gaussian 
sketch. First, we observe that the theoretical predictions closely
align with empirical values even after the sketch size crosses the
stable rank $r\approx\frac1{1-\alpha}$, suggesting that Theorem~\ref{t:main} can
be extended to this regime. Second, while it is not surprising that
the error decays at a similar 
rate as the singular values, our predictions offer a much more precise
description, down to lower order effects and
even constant factors. For instance, we observe that the error
(normalized by $\|\A\|_F^2$, as in the figure) only
starts decaying exponentially after $k$ crosses the stable rank, and
until that point it decreases at a linear rate with slope
$-\frac{1-\alpha}{2\sqrt\alpha}$.

\subsection{Polynomial spectral decay}
We now turn to polynomial spectral decay, which is a natural model for
analyzing heavy-tailed data distributions. Let $\A$ have squared
singular values $\sigma_i^2=C\cdot i^{-\beta}$ for some $\beta\geq 2$, and let
$m,n\rightarrow\infty$. As in the case of exponential decay,
we use the integral $\int_y^\infty\!\frac1{1+(C\gamma)^{-1}x^{-\beta}}dx$
to approximate the sum in \eqref{eq:equation}, and 
solve for $\gamma$, obtaining $\gamma \approx
\big((k+\frac12)\frac\beta\pi\sin(\frac\pi\beta)\big)^\beta$. Combining
this with Corollary \ref{c:low-rank} we get:
\begin{align}
  \E\big[\|\A-\A\P\|_F^2\big] \approx
  C\cdot\frac{k}{(k+\frac12)^\beta}\bigg(\frac{\pi/\beta}{\sin(\pi/\beta)}\bigg)^\beta,
  \quad\text{when}\quad\sigma_i^2=C\cdot
  i^{-\beta}\ \ \text{for all $i$.}\label{eq:explicit-poly}
\end{align}

Figure \ref{f:explicit}b compares our predictions to the empirical
results for several values of $\beta$. In all of these cases, the
stable rank is close to 1, and yet the theoretical predictions align
very well with the empirical results. Overall, the asymptotic rate of decay of the error is
$k^{1-\beta}$.  However it is easy to verify that the lower order
effect of $(k+\frac12)^\beta$ appearing instead of $k^\beta$
in \eqref{eq:explicit-poly} significantly changes the trajectory for
small values of $k$. Also, note that as $\beta$ grows large, the
constant $\big(\frac{\pi/\beta}{\sin(\pi/\beta)}\big)^\beta$ goes to
$1$, but it plays a significant role for $\beta=2$ or $3$
(roughly, scaling the expression by a factor of $2$). Finally, we
remark that for $\beta\in(1,2)$, our integral approximation of
\eqref{eq:equation} becomes less accurate. We expect that a corrected
expression is possible, but likely more complicated and less interpretable.
\vspace{3mm}

\ifisarxiv
\else
\begin{minipage}{.48\textwidth}
\hspace{-4.5mm}\includegraphics[width=.57\textwidth]{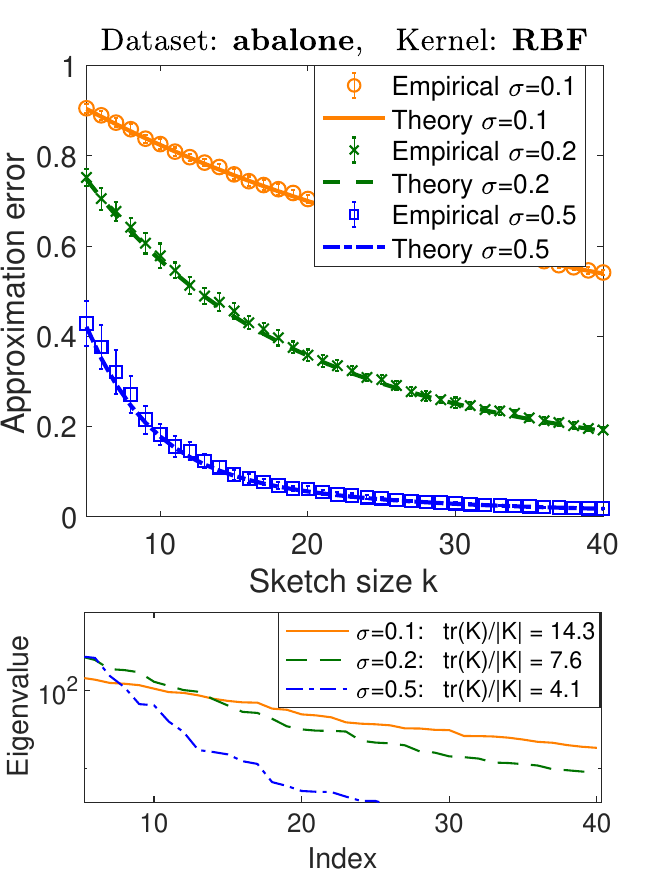}\nobreak\hspace{-2.5mm}\includegraphics[width=.57\textwidth]{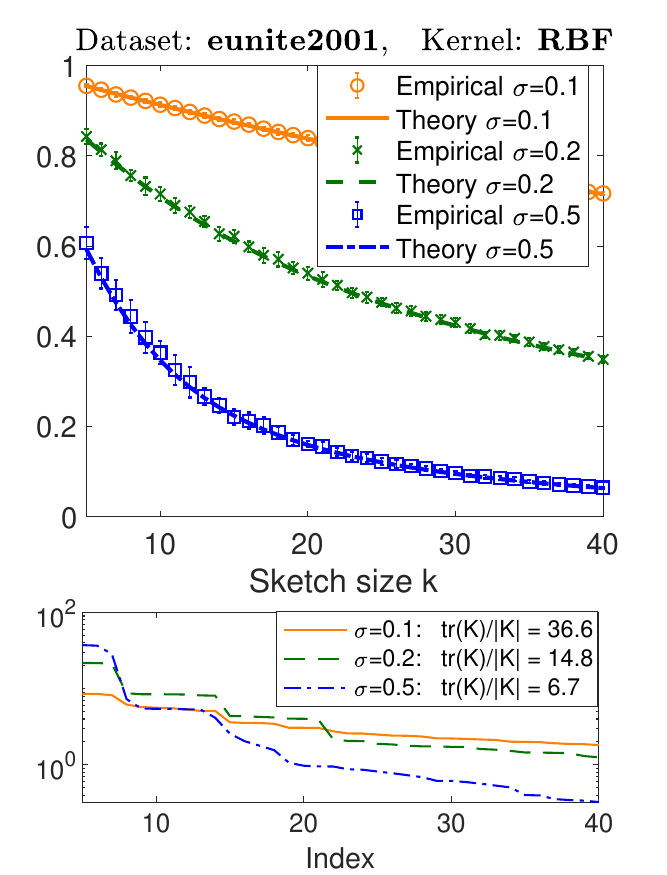}\vspace{-1mm}
  \captionof{figure}{Theoretical predictions versus approximation error for the
    sketched Nystr\"om with the RBF kernel
    (spectral decay shown at the bottom).}\label{f:nystrom}
\end{minipage}\hspace{4mm}
\begin{minipage}{.48\textwidth}
\hspace{-2.5mm}\includegraphics[width=.57\textwidth]{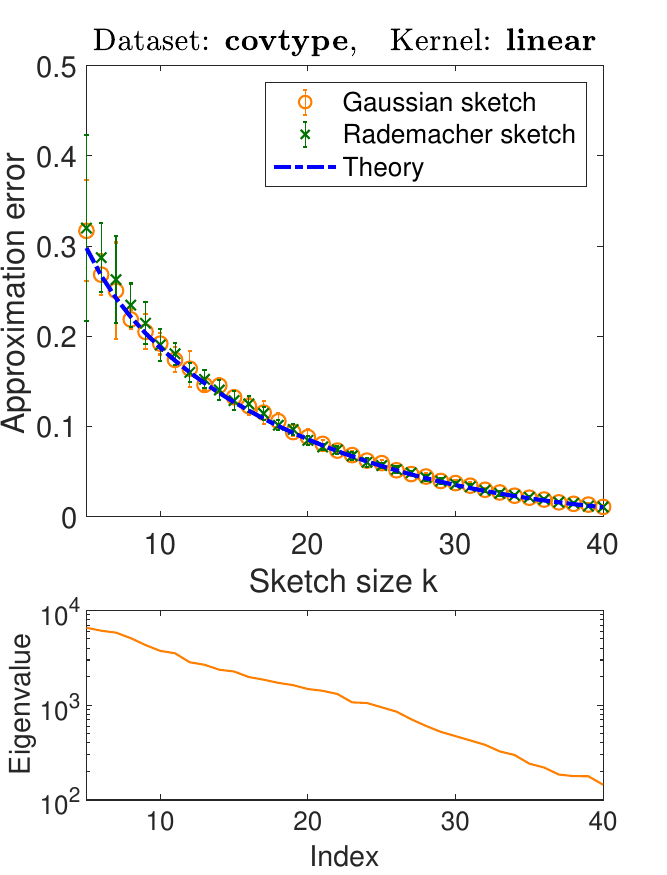}\nobreak\hspace{-2.5mm}\includegraphics[width=.57\textwidth]{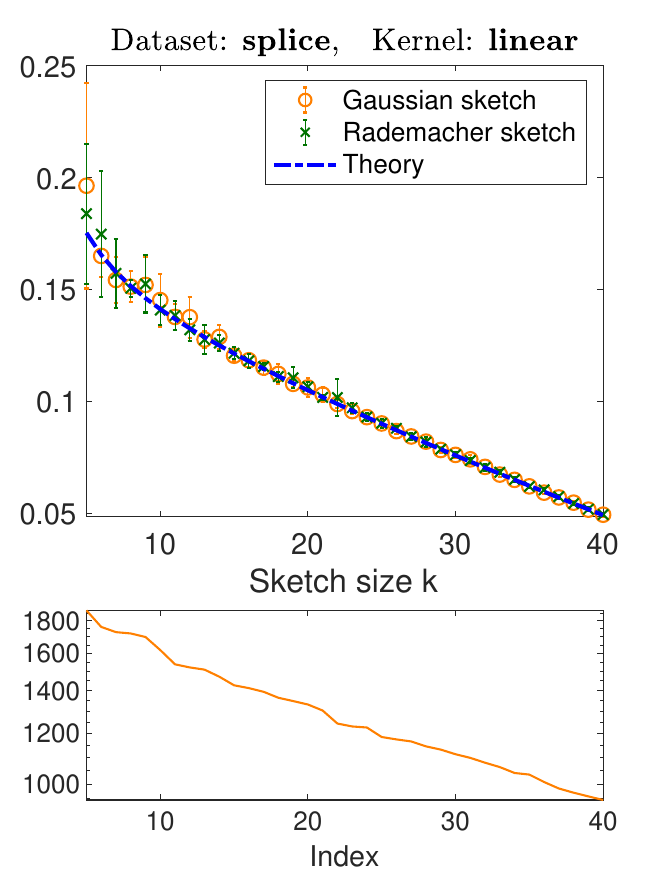}\vspace{-1mm}
  \captionof{figure}{Theoretical predictions versus approximation error for
the Gaussian and Rademacher sketches
  (spectral decay shown at the bottom).}\label{f:linear}
\end{minipage}
\fi

\section{Empirical results}\label{s:experiments}

In this section, we numerically verify the accuracy of our theoretical
predictions for the low-rank approximation error of sketching on
benchmark datasets from the libsvm repository 
\cite{libsvm} (further numerical results are in
Appendix~\ref{a:experiments}). We repeated every experiment 10 times,
and plot 
both the average and standard deviation of the results. We use the
following $k\times m$ sketching matrices $\S$:\vspace{-1mm}
\begin{enumerate}
\item \emph{Gaussian sketch:} with i.i.d.\@ standard normal entries;
\item \emph{Rademacher sketch:} with
  i.i.d.\@ entries equal $1$ with probability 0.5 and $-1$ otherwise.
\end{enumerate}
\vspace{-2mm}

\ifisarxiv
\begin{figure}
  \centering
\includegraphics[width=.43\textwidth]{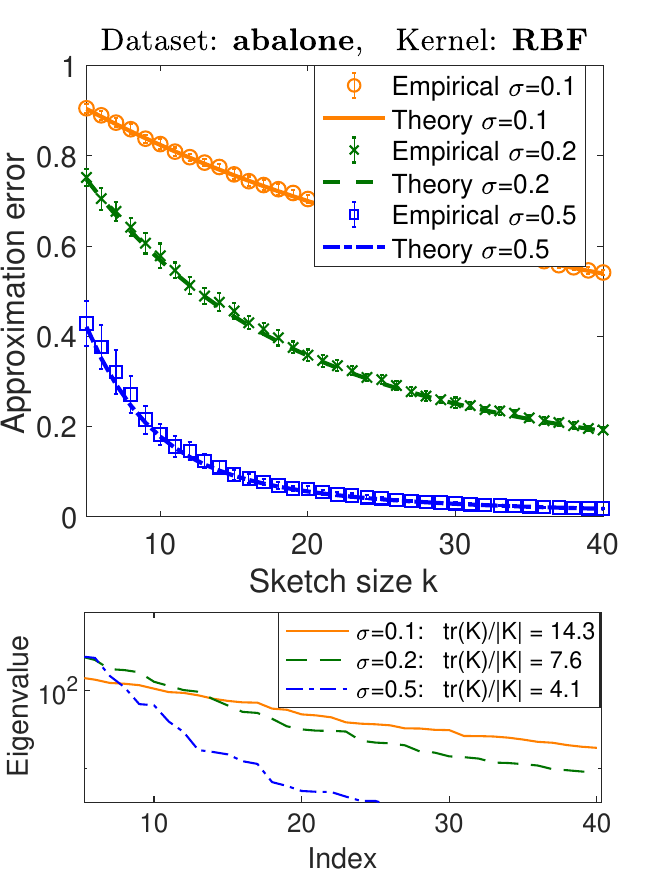}\nobreak\includegraphics[width=.43\textwidth]{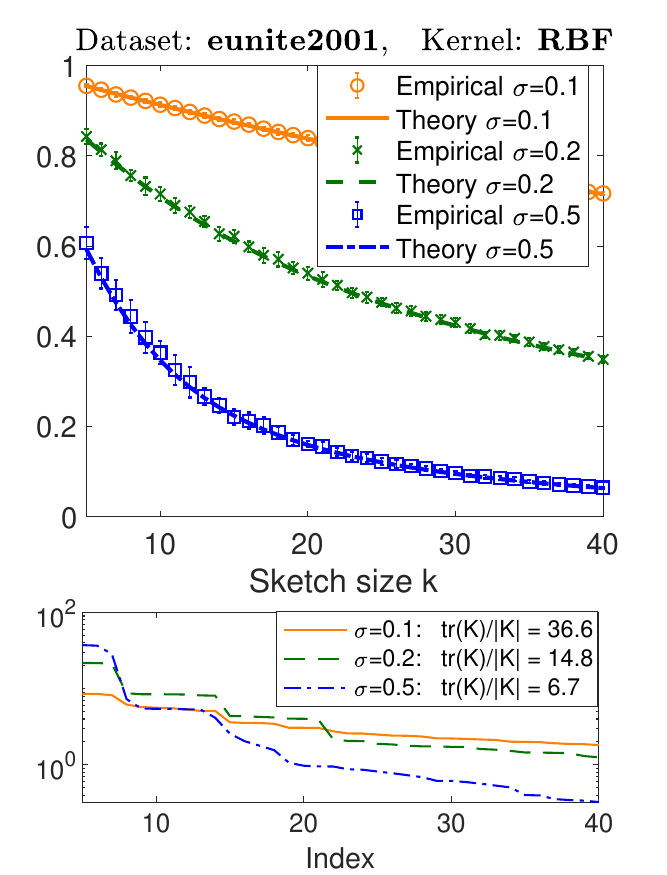}
  \caption{Theoretical predictions versus approximation error for the
    sketched Nystr\"om with the RBF kernel
    (spectral decay shown at the bottom).}\label{f:nystrom}
\end{figure}
\fi

\ifisarxiv
\begin{figure}
  \centering
\includegraphics[width=.43\textwidth]{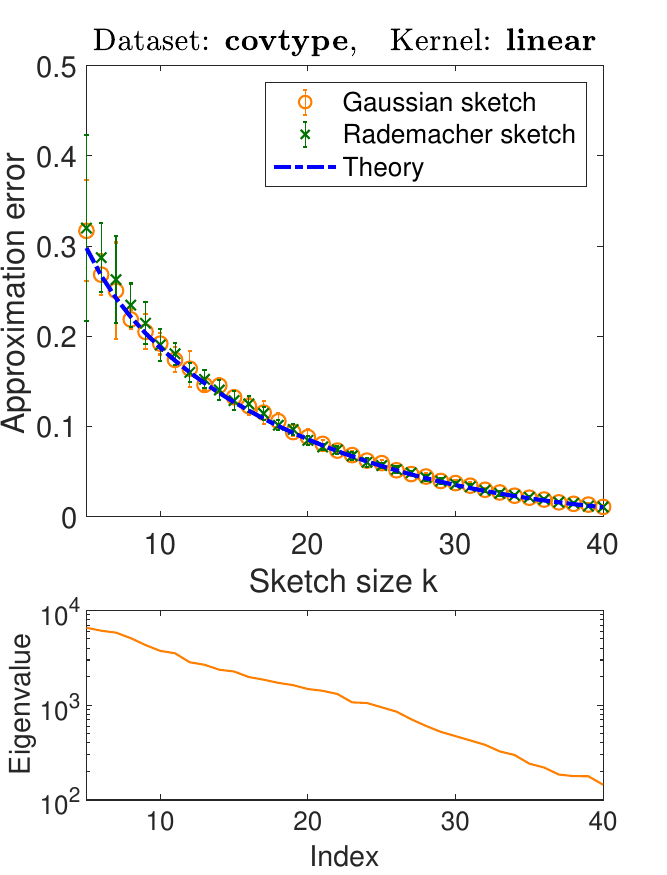}\nobreak\includegraphics[width=.43\textwidth]{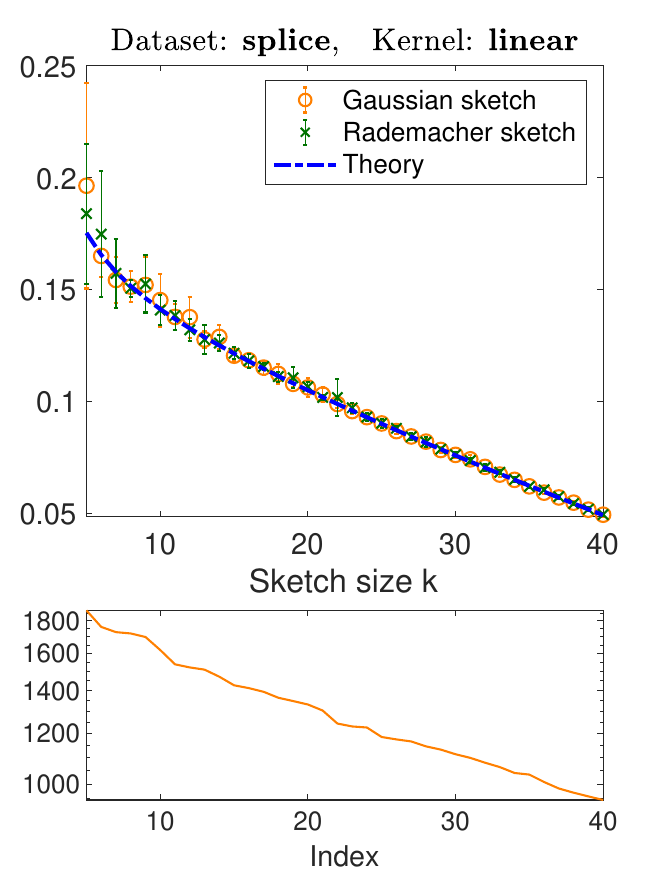}
  \caption{Theoretical predictions versus approximation error for
the Gaussian and Rademacher sketches
  (spectral decay shown at the bottom).}\label{f:linear}
\end{figure}
\fi

\paragraph{Varying spectral decay.}  To demonstrate the role of spectral
decay and the stable rank on the approximation error, we performed
feature expansion using the radial basis function (RBF) kernel
$k(\a_i,\a_j)=\exp(-\|\a_i-\a_j\|^2/(2\sigma^2))$, obtaining an $m\times m$
kernel matrix $\K$. We used the sketched Nystr\"om method to construct
a low-rank approximation $\tilde
\K=\K\S^\top(\S\K\S^\top)^\dagger\S\K$, and computed the normalized trace norm 
error $\|\K-\tilde\K\|_*/\|\K\|_*$. The theoretical predictions are coming from
\eqref{eq:nystrom}, which in turn uses Theorem \ref{t:main}.
Following \cite{revisiting-nystrom}, we use the RBF kernel because
varying the scale parameter $\sigma$ 
allows us to observe the approximation error under qualitatively
different spectral decay profiles of the kernel. In Figure \ref{f:nystrom},
we present the results for the Gaussian sketch on two datasets, with three values
of $\sigma$, and in all cases our theory aligns with the empirical
results. Furthermore, as smaller $\sigma$ leads to slower spectral decay and 
larger stable rank, it also makes the approximation error decay more
linearly for small sketch sizes. This behavior is predicted by our explicit
expressions \eqref{eq:explicit-exp} for the error under exponential spectral decay from
Section \ref{s:explicit}. Once the sketch sizes are sufficiently
larger than the stable rank of $\K^{\frac12}$, the error starts decaying at an
exponential rate. Note that Theorem \ref{t:main} only guarantees
accuracy of our expressions for sketch sizes below the stable rank,
however the predictions are accurate regardless of this constraint.
\vspace{-2mm}
\paragraph{Varying sketch type.} In the next set of
empirical results, we compare the performance of Gaussian and Rademacher
sketches, and also verify the theory when sketching the data matrix $\A$
without kernel expansion, plotting
$\|\A-\A(\S\A)^\dagger\S\A\|_F^2/\|\A\|_F^2$.
Since both of the sketching methods have 
sub-gaussian entries, Corollary~\ref{c:low-rank} predicts that they
should have comparable performance in this task and match our
expressions. This is exactly what we observe in Figure \ref{f:linear}
for two datasets and a range of sketching sizes, as well as in other
empirical results shown in Appendix \ref{a:experiments}.

\section{Conclusions}

We derived the first theoretically supported precise expressions for the
expected residual projection matrix, which is a central component in
the analysis of RandNLA dimensionality reduction via sketching. Our
analysis provides a new understanding of low-rank approximation, the Nystr\"om method, and the
convergence properties of many randomized iterative algorithms. As a direction for future work, we conjecture that our
main result can be extended to sketch sizes larger than the stable
rank of the data matrix.

\ifisarxiv\else
\section*{Broader Impact}
In this paper, we investigate the spectral properties of residual
(random) projection matrices, commonly appearing in various
sketching-based methods. The precise theoretical description given in
this paper provides performance guarantees for popular algorithms
such as low-rank approximation and many randomized (iterative)
optimization methods, and contributes to the development of more
robust and reliable large-scale learning systems. The
theoretical framework developed in this work presents no foreseeable
negative societal consequence. 
\fi

\paragraph{Acknowledgments.}
We would like to acknowledge DARPA, IARPA, NSF, and ONR via its BRC on
RandNLA for providing partial support of this work.  Our conclusions
do not necessarily reflect the position or the policy of our sponsors,
and no official endorsement should be inferred.

\bibliographystyle{alpha}
\bibliography{../pap}

\ifisarxiv\else\newpage\fi
\appendix

\section{Proof of Theorem~\ref{t:main-tech}}
\label{sec:proof-of-theo-main-tech}

We first introduce the following technical lemmas.

\begin{lemma}\label{l:rank-one-update}
    For $\X \in \mathbb R^{k \times n}$ with $k<n$, denote $\P = \X^\dagger \X$ and $\P_{-k} = \X_{-k}^\dagger \X_{-k}$, with $\X_{-i} \in \mathbb R^{(k-1) \times n}$ the matrix $\X$ without its i-th row $\x_i \in \mathbb R^n$. Then, conditioned on the event $E_k: \left\{ \left| \frac{\tr \Sigmab (\I - \P_{-k})}{\x_k^\top (\I - \P_{-k}) \x_k} - 1 \right| \le \frac12 \right\}$:
    \begin{align*}
      (\X^\top\X)^\dagger\x_k = \frac{(\I - \P_{-k}) \x_k}{\x_k^\top (\I - \P_{-k}) \x_k}, \quad \P - \P_{-k} = \frac{(\I - \P_{-k}) \x_k \x_k^\top (\I - \P_{-k}) }{\x_k^\top (\I - \P_{-k}) \x_k}.
    \end{align*}
\end{lemma}

\begin{proof}
Since conditioned on $E_k$ we have $\x_k^\top (\I - \P_{-k}) \x_k \neq 0$, from \cite[Theorem 1]{10.2307/2099767} we deduce
\begin{align*}
  (\X^\top \X)^\dag = \left( \A + \x_k \x_k^\top \right)^\dagger
  &= \A^\dag
  - \frac{\A^\dag \x_k \x_k^\top (\I - \P_{-k}) }{\x_k^\top ( \I - \P_{-k}) \x_k}
  - \frac{( \I -\P_{-k}) \x_k \x_k^\top \A^\dag}{\x_k^\top ( \I - \P_{-k}) \x_k} \\ 
  &+ (1 + \x_k^\top \A^\dag \x_k) \frac{(\I - \P_{-k}) \x_k \x_k^\top (\I - \P_{-k})}{(\x_k^\top (\I - \P_{-k}) \x_k)^2}
\end{align*}
for $\A = \X_{-k}^\top \X_{-k}$ so that $\I - \P_{-k} = \I - \A^\dagger \A$, where we used the fact that $\I - \P_{-k}$ is a projection matrix so that $(\I - \P_{-k})^2 = \I - \P_{-k}$. As a consequence, multiplying by $\x_k$ and simplifying we get
\[
  (\X^\top \X)^\dag \x_k = \frac{ ( \I -\P_{-k}) \x_k}{\x_k^\top (\I - \P_{-k}) \x_k}.
\]
By definition of the pseudoinverse, $\P = \X^\dagger \X = (\X^\top \X)^\dagger \X^\top \X$ so that 
\[
  \P - \P_{-k} = \X^\dagger \X - \X_{-k}^\dagger \X_{-k} = \frac{( \I - \P_{-k}) \x_k \x_k^\top ( \I - \P_{-k})}{\x_k^\top ( \I - \P_{-k}) \x_k}
\]
where we used $\A (\I - \P_{-k}) = \A - \A \A^\dagger \A = 0$ and thus the conclusion.
\end{proof}

\begin{lemma}
\label{l:concentration-sub-gaussian}
If random vector $\x \in \mathbb R^n$, with $\E[\x] = 0$ and $\E[\x
\x^\top] = \I_n$, satisfies the Hanson-Wright inequality with constant
$K\geq 1$ (Definition \ref{d:hanson-wright}), then for any positive
semi-definite matrix $\A \in \mathbb R^{n \times n}$, we have 
\[
  \Pr\left[
    \lvert \x^\top \A \x - \tr \A \rvert
    \geq \frac13 \tr \A
    \right] \leq 2 \exp\left( - \frac{r_{\!\A}}{C K^4} \right)
\]
with $r_{\!\A} = \tr \A/\| \A \|$ the stable rank of $\A$, and 
\begin{align*}
  \E\left[ \left( \x^\top \A \x - \tr \A \right)^2 \right]
  &\leq c~K^4~\tr \A^2
\end{align*}
for some $C,c > 0$ independent of $K$. Also, for any unit vector
$\v\in\R^n$, variable $\v^\top\x$ is $CK$-sub-gaussian.
\end{lemma}

\begin{proof}
  Recall that, for a psd matrix $\A$, the Hanson-Wright inequality can
  be stated as follows:
\[
  \Pr \left\{ \left| \x^\top \A \x - \tr \A  \right| \ge t \right\} \le 2 \exp \left( - \min \left\{ \frac{t^2}{ K^4 \tr\A^2 },  \frac{t}{K^2 \|\A\|  } \right\} \right).
\]
Taking $t = \frac13 \tr \A$ we have
\[
  \frac{t^2}{K^4 \tr\A^2 } = \frac{ (\tr \A)^2 }{9 K^4 \tr\A^2 } \ge
  \frac{ \tr \A }{9 K^4 \| \A \| } = \frac{r_{\!\A}}{9 K^4},
  \qquad \frac{t}{K^2 \|\A\|  } \ge \frac{r_{\!\A}}{3 K^2},
\]
where we use the fact that $\tr \A^2 \le \| \A \| \tr \A$.

Integrating this bound yields:
\[
  \E\left[ (\x^\top \A \x - \tr \A )^2 \right] \le c~K^4~\tr \A^2.
\]
Finally, to obtain the sub-gaussianity of $\v^\top\x$, we apply the Hanson-Wright inequality to $\A=\v\v^\top$.
\end{proof}

\begin{lemma}\label{l:trace}
With the notations of Lemma~\ref{l:rank-one-update}, for $X=\tr\,\Sigmab(\P_{-k}-\E[\P_{-k}])$
and $\|\Sigmab\| = 1$, we have
\begin{align*}
  \E[X^2]\leq
  Ck\quad\text{and}\quad\Pr\{|X|\geq t\}\leq 2 \ee^{-\frac{t^2}{ck}}.
  \end{align*}
  for some universal constant $C,c > 0$.
\end{lemma}
\begin{proof}
To simplify notations, we work on $\P$ instead of $\P_{-k}$, the same line of argument applies to $\P_{-k}$ by changing the sample size $k$ to $k-1$.

First note that
\begin{align*}
  X
  &= \tr \Sigmab (\P - \E \P)
  = \E_k [\tr \Sigmab \P]  - \E_0[ \tr \Sigmab \P ]\\
  &= \sum_{i=1}^k \left( \E_i[\tr \Sigmab \P] - \E_{i-1}[\tr \Sigmab \P]\right)
  = \sum_{i=1}^k (\E_i - \E_{i-1}) \tr \Sigmab (\P - \P_{-i})
\end{align*}
where we used the fact that $\E_i [\tr \Sigmab \P_{-i}] = \E_{i-1} [\tr \Sigmab \P_{-i}]$, for $\E_i[\cdot]$ the conditional expectation with respect to $\mathcal F_i$ the $\sigma$-field generating the rows $\x_1 \ldots, \x_i$ of $\X$. This forms a martingale difference sequence (it is a difference sequence of
the Doob martingale for $\tr \Sigmab (\P - \P_{-i})$ with respect to filtration $\mathcal{F}_i$)
hence it falls within the scope of the Burkholder inequality \cite{burkholder1973distribution}, recalled as follows.

\begin{lemma}\label{l:burkholder}
For $\{ x_i \}_{i=1}^k$ a real martingale difference sequence with respect to the increasing $\sigma$ field $\mathcal F_i$, we have, for $L > 1$, there exists $C_L > 0$ such that
\[
  \E \bigg[
    \Big| \sum_{i=1}^k x_i \Big|^L
  \bigg]
  \le C_L \E \bigg[
    \Big( \sum_{i=1}^k |x_i|^2 \Big)^{L/2}
  \bigg].
\]
\end{lemma}

From Lemma~\ref{l:rank-one-update}, $\P - \P_{-i} = \frac{(\I - \P_{-i}) \x_i \x_i^\top (\I - \P_{-i}) }{ \x_i^\top (\I - \P_{-i}) \x_i}$ is positive semi-definite, we have $\tr \Sigmab (\P - \P_{-i}) \le \| \Sigmab \| = 1$ so that with Lemma~\ref{l:burkholder} we obtain with $x_i = (\E_i - \E_{i-1}) \tr \Sigmab (\P - \P_{-i})$ that, for $L > 1$
\[
  \E |X|^L \le C_L k^{L/2}.
\]
In particular, for $L=2$, we obtain $\E |X|^2 \le C k$.

For the second result, since we have almost surely bounded martingale differences
($\lvert x_i \rvert \leq 2$), by the Azuma-Hoeffding inequality
\begin{align*}
  \Pr\{
    \lvert X \rvert
    \geq t
  \}
  \leq 2 \ee^{\frac{-t^2}{8 k}}
\end{align*}
as desired.
\end{proof}

\subsection{Complete proof of Theorem~\ref{t:main-tech}}

Equipped with the lemmas above, we are ready to prove Theorem~\ref{t:main-tech}. First note that:
\begin{enumerate}[leftmargin=*]
  \item  Since $\X^\dag \X \overset{d}{=} (\alpha \X)^\dag (\alpha \X)$ for any $\alpha \in \R \setminus \{0\}$, we can assume without loss of generality (after rescaling $\bar\P_\perp$ correspondingly) that $\|\Sigmab\| = 1$.
  \item According to the definition of $\bar \P_\perp$ and $\gamma$, the following bounds hold
  \begin{equation}\label{eq:bound-bar-gamma-P}
    \frac1{\gamma + 1} \I \preceq \bar \P_\perp \preceq \I, \quad \gamma \le \frac{k}{r-k} = \frac1{\rho - 1}
  \end{equation}
  for $r \equiv \frac{\tr \Sigmab}{ \| \Sigmab \| } = \tr \Sigmab$ and $\rho \equiv \frac{r}k > 1$, where we used the fact that
  \[
    k = n - \tr\,\bar\P_\perp = \tr\,\bar\P_\perp (\gamma\Sigmab + \I) - \tr\,\bar\P_\perp
  = \gamma \tr\,\bar\P_\perp \Sigmab\geq \frac{\gamma}{\gamma+1}\tr\,\Sigmab,
  \]
  so that $r = \tr \Sigmab \le k \cdot \frac{\gamma + 1}{\gamma}$.
  \item As already discussed in Section~\ref{subsec:proof-sketch}, to obtain the lower and upper bound for $\E[\P_\perp]$ in the sense of symmetric matrix as in Theorem~\ref{t:main-tech}, it suffices to bound the following spectral norm
  \begin{equation}\label{eq:proof-1-SM}
    \| \I - \E[\P_\perp] \bar \P_\perp^{-1} \| \le \frac{C_\rho}{\sqrt r}, 
  \end{equation}
  so that, with $\frac{\rho-1}{\rho}\I \preceq \bar\P_\perp \preceq \I$ from \eqref{eq:bound-bar-gamma-P}, we have
  \begin{align*}
    \|\I-\bar\P_\perp^{-\frac12}\E[\P_\perp]\bar\P_\perp^{-\frac12}\| =
    \|\bar\P_\perp^{-\frac12}(\I-\E[\P_\perp]\bar\P_\perp^{-1})\bar\P_\perp^{\frac12}\|\leq
  \frac{C_\rho}{\sqrt r} \sqrt{\frac{\rho}{\rho-1}}.
  \end{align*}
  Defining $\epsilon =\frac{C_\rho}{\sqrt r}
  \sqrt{\frac{\rho}{\rho-1}}$, this means that all eigenvalues 
  of the p.s.d.~matrix $\bar\P_\perp^{-\frac12}\E[\P_\perp]\bar\P_\perp^{-\frac12}$ lie
  in the interval $[1-\epsilon,1+\epsilon]$, and
  \begin{align*}
    (1-\epsilon)\I\preceq
    \bar\P_\perp^{-\frac12}\E[\P_\perp]\bar\P_\perp^{-\frac12}\preceq (1+\epsilon)\I.
  \end{align*}
  so that by multiplying $\bar \P_{\perp}^{\frac12}$ on both sides, we obtain the
  desired bound.
\end{enumerate}

\bigskip

As a consequence of the above observations, we only need to prove \eqref{eq:proof-1-SM} under the setting $\| \Sigmab \| = 1$. The proof comes in the following two steps:

\begin{enumerate}[leftmargin=*]
  \item For $\P_{-i} = \X_{-i}^\dagger \X_{-i}$, with $\X_{-i} \in \mathbb R^{(k-1) \times n}$ the matrix $\X$ without its $i$-th row, we define, for $i \in \{ 1, \ldots, k \}$, the following events
  \begin{equation}\label{eq:def-events-E-F}
    E_i: \left\{ \left| \frac{\tr (\I - \P_{-i}) \Sigmab}{\x_i^\top ( \I - \P_{-i}) \x_i} - 1 \right| \le \frac12 \right\}, 
  \end{equation}
  where we recall $\x_i \in \R^n$ is the $i$-th row of $\X$ so that $\E[\x_i] = 0$ and $\E[\x_i \x_i^\top] = \Sigmab$. With Lemma~\ref{l:concentration-sub-gaussian}, we can bound the probability of $\neg E_i$, and consequently that of $\neg E$ for $E = \bigwedge_{i=1}^k E_i$;
  \item We then bound, conditioned on $E$ and $\neg E$ respectively, the spectral norm $\| \I - \E[\P_\perp] \bar \P_\perp^{-1} \|$. More precisely, since
\begin{align*}
  \I-\E[\P_\perp]\bar\P_\perp^{-1}\!
    &= \E[\P] - \gamma\E[\P_\perp]\Sigmab
= \E[\P \cdot \one_E] +
    \E[\P  \cdot \one_{\neg E}]
    -\gamma\E[\P_\perp]\Sigmab\\
  &=k\,\E\bigg[\frac{(\I - \P_{-k})\x_k\x_k^\top}{\x_k^\top(\I - \P_{-k})\x_k} \cdot \one_E\bigg]
    -\gamma\E[\P_\perp]\Sigmab \ +\     \E[\P \cdot \one_{\neg E}]\\
  &=\gamma\underbrace{\E\bigg[(\bar s-\hat
    s)\,\frac{(\I -\P_{-k})\x_k\x_k^\top}{\x_k^\top(\I -\P_{-k})\x_k} \cdot \one_E\bigg]}_{\T_1}
    - \gamma \underbrace{\E[ (\I - \P_{-k}) \x_k\x_k^\top \cdot \one_{\neg E}]}_{\T_2} 
    + \gamma \underbrace{\E[\P-\P_{-k}]\Sigmab}_{\T_3} + \underbrace{\E[\P \cdot \one_{\neg E}]}_{\T_4},
\end{align*}
where we used Lemma~\ref{l:rank-one-update} for the third equality and denote $\hat s = \x_k^\top ( \I - \P_{-k}) \x_k$ as well as $\bar s = \tr \bar \P_\perp \Sigmab =k/\gamma$. It then remains to bound the spectral norms of $ \T_1, \T_2, \T_3 ,\T_4$ to reach the conclusion.
\end{enumerate}

Another important relation that will be constantly used throughout the proof is
\begin{equation}\label{eq:bound-trace-P-Sigma}
    \tr (\I - \P_{-k}) \Sigmab = \tr \Sigmab^{\frac12} (\I - \P_{-k})^2 \Sigmab^{\frac12} =
    \|\Sigmab^{\frac12}-\Sigmab^{\frac12}\X_{-k}^\dagger\X_{-k}\|_F^2\geq
\sum_{i\geq k}\lambda_i (\Sigmab)\geq r-k
\end{equation}
where we used the fact that $\rank(\X_{-k}^\dagger \X_{-k}) \le \rank(\X_{-k}) \le k-1$ and arranged the eigenvalues $ 1 = \lambda_1(\Sigmab) \ge \ldots \ge \lambda_n(\Sigmab)$ in a non-increasing order. As a consequence, we also have
\begin{equation}\label{eq:bound-norm-P-Sigma}
  \frac{\tr (\I - \P_{-k}) \Sigmab}{ \| \Sigmab^{\frac12} (\I - \P_{-k}) \Sigmab^{\frac12} \| } \ge \tr (\I - \P_{-k}) \Sigmab \geq r-k.
\end{equation}

\bigskip

  For the first step, assuming without loss of generality that that the
Hanson-Wright constant satisfies $K\geq 1$ and using
Lemma~\ref{l:concentration-sub-gaussian} together with
\eqref{eq:bound-norm-P-Sigma}, we have:
\begin{align*}
  \Pr (\neg E_i) &\le \Pr \left\{ \lvert \x_i^\top (\I - \P_{-i}) \x_i - \tr \Sigmab (\I - \P_{-i})  \rvert
    \geq \frac13 \tr \Sigmab (\I - \P_{-i}) \right\} \le 2 \ee^{ -\frac{ r -k }{C K^4}},
\end{align*}
so that with the union bound we obtain
\begin{equation}\label{eq:union-bound-neg-E}
  \Pr (\neg E) \le 2k \ee^{ -\frac{ r -k }{C K^4}} \le \frac{k}{(r-k)^2} \cdot 2 (r-k)^2 \ee^{ - \frac{ r -k }{C K^4} } \le \frac{C_\rho}{r - k},
\end{equation}
where we used the fact that, for $\alpha,x >0$, we have $x^2 \ee^{-\alpha x} \le \frac{4 \ee^{-2}}{\alpha^2}$. Also, denoting $c_\rho = \frac{r - k}r = \frac{\rho - 1}{\rho} > 0$, we have
\begin{equation}\label{eq:proba-E}
  \Pr(\neg E) \le \frac{C_\rho}{r - k} = \frac{C_\rho}{c_\rho r} = \frac{C'_\rho}{r} 
\end{equation}
for some $C'_\rho > 0$ that depends on $\rho = r/k > 1$ and the
Hanson-Wright constant $K$.

\medskip

At this point, note that, conditioned on the event $E$, we have for $i \in \{1, \ldots, k \}$
\begin{equation}
    \frac12 \frac1{\tr (\I - \P_{-i}) \Sigmab} \le \frac1{\x_i^\top (\I - \P_{-i}) \x_i} \le \frac32 \frac1{\tr (\I - \P_{-i}) \Sigmab}, 
\end{equation}
Also, with \eqref{eq:proba-E} and the fact that $\| \P \| \le 1$, we have $ \| \T_4 \| \le
\frac{C_\rho}{r}$ for some $C_\rho > 0$ that depends on $\rho$ and $K$. To handle non-symmetric matrix $\T_2$, note that $\T_2 + \T_2^\top$ is symmetric and
\begin{equation}
  - \E[ (\I - \P_{-k}) \cdot \one_{\neg E}] - \E[ (\x_k^\top \x_k) \x_k\x_k^\top \cdot \one_{\neg E}] \preceq \T_2 + \T_2^\top \preceq \E[ (\I - \P_{-k}) \cdot \one_{\neg E}] + \E[ (\x_k^\top \x_k) \x_k\x_k^\top \cdot \one_{\neg E}]
\end{equation}
with $-(\A \A^\top + \B \B^\top) \preceq \A \B^\top + \B \A^\top \preceq \A \A^\top + \B \B^\top$. To obtain an upper bound for operator norm of $\E[ (\x_k^\top \x_k) \x_k\x_k^\top \cdot \one_{\neg E}]$, note that 
\begin{align*}
  \| \E[ (\x_k^\top \x_k) \x_k\x_k^\top \cdot \one_{\neg E}] \| &\le \E[ (\x_k^\top \x_k)^2 \cdot \one_{\neg E}] = \int_0^\infty \Pr ( \x^\top \x \cdot \one_{\neg E} \ge \sqrt t ) dt \\ 
  &\le \int_0^X \Pr ( \x^\top \x \cdot \one_{\neg E} \ge \sqrt t ) dt + \int_X^\infty \Pr ( \x^\top \x \ge \sqrt t ) dt  \\ 
  &\le X \cdot \Pr(\neg E) + \int_X^\infty \ee^{- \min \left\{ \frac{t}{C^2 K^4 r}, \frac{\sqrt t}{CK^2 \sqrt r} \right\} } dt \le \frac{c_\rho}r
\end{align*}
where we recall $\E[\x^\top \x] = \tr \Sigmab = r$ and take $X \ge C^2 K^4 r$, the third line follows from the proof of Lemma~\ref{l:concentration-sub-gaussian} and the forth line from the same argument as in \eqref{eq:union-bound-neg-E}. Moreover, since $\| \T_2 \| \le \| \T_2 + \T_2^\top \|$ (see for example \cite[Proposition~5.11]{serre2010matrices}), we conclude that $\| \T_2 \| \le \frac{C_\rho}r$.

And it thus remains to handle the terms $\T_1$ and $\T_3$ to obtain a bound on $\| \I - \E[\P_\perp] \bar \P_\perp^{-1} \| $. 

To bound $\T_3$, with $\P - \P_{-k} = \frac{(\I - \P_{-k})\x_k\x_k^\top (\I - \P_{-k})}{\x_k^\top (\I - \P_{-k}) \x_k}$ in Lemma~\ref{l:rank-one-update}, we have
\begin{align*}
\|\T_3\| &\leq  \bigg\|
  \E\bigg[\frac{ (\I - \P_{-k}) \x_k\x_k^\top (\I - \P_{-k})}{\x_k^\top (\I - \P_{-k}) \x_k}
  \cdot \one_E \bigg]\bigg\|  + \| \E[ (\P - \P_{-k}) \cdot \one_{\neg E} ] \| \\ 
  &\le \frac32 \E \left[ \frac1{ \tr (\I - \P_{-k} ) \Sigmab } \right] + \frac{c_\rho}{r - k} \leq \frac{C_\rho}{r-k} = \frac{C'_\rho}{r}
\end{align*}
where we used the fact that $\tr\,(\I - \P_{-k}) \Sigmab \ge r -k $ from \eqref{eq:bound-trace-P-Sigma} and recall $\rho \equiv r/k > 1$.

For $\T_1$ we write
\begin{align*}
  \|\T_1\|    &\leq\E\bigg[\|\I - \P_{-k}\|\cdot \Big\|\E\Big[|\bar s-\hat s|\cdot
    \frac{\x_k\x_k^\top}{\x_k^\top( \I - \P_{-k}) \x_k} \cdot \one_E\mid\P_{-k}\Big]
    \Big\|\bigg]
  \\
          &\leq \frac32 \frac1{r-k} \cdot \E\bigg[\sup_{\| \v \| = 1}\E\Big[
            |\bar s-\hat s|\cdot \v^\top\x_k\x_k^\top\v\cdot\one_E\mid
            \P_{-k}\Big]\bigg]
  \\
  &\leq \frac{C_\rho}{r}\cdot \E\bigg[\underbrace{\sqrt{\E\big[(\bar s-\hat
    s)^2\cdot \one_E \mid\P_{-k}\big]}}_{ T_{1,1} }\cdot
    \underbrace{\sup_{ \| \v \| = 1 }\sqrt{\E\big[(\v^\top\x_k)^4\big]}}_{ T_{1,2} }\bigg]
\end{align*}
where we used Jensen's inequality for the first inequality, the relation in \eqref{eq:bound-trace-P-Sigma} for the second inequality, and Cauchy–Schwarz for the third inequality.

  We first bound $T_{1,2}$ by definition of sub-gaussian random
vectors. From Lemma \ref{l:concentration-sub-gaussian}, for any unit vector $\v$, we have that
$\v^\top \x_k$ is a $O(K)$-sub-gaussian random variable. As such,
$T_{1,2} \le C K^2$ for some absolute constant $C > 0$, see for example
\cite[Section~2.5.2]{vershynin2018high}. 

For $T_{1,1}$ we have
\begin{align*}
  \sqrt{ \E \big[(\bar s-\hat s)^2 \cdot\one_E \mid\P_{-k}\big] } = \sqrt{ (\bar s - s)^2 + \E\big[(s - \hat s)^2 \cdot\one_E\big] }
\end{align*}
where we denote $s=\E[\hat s]=\tr\,\E[\I - \P_{-k}]\Sigmab$. Note that
\begin{align*}
  \E\big[(s - \hat s)^2 \big] &=
\E\big[\big(\tr\,\Sigmab(\P_{-k} - \E[\P_{-k}])\big)^2  \big] +
  \E\big[(\tr\,(\I - \P_{-k})\Sigmab-\x_k^\top(\I - \P_{-k})\x_k)^2  \big]\\
                              &\leq C_1 k +  C_2 \E\big[\tr\,(\Sigmab - \P_{-k}\Sigmab)^2 \big] \\
                              &\le C(k+s)\\
                              &\leq C\big(k+\bar s + |s- \bar s| \big)
\end{align*}
where we used Lemma~\ref{l:trace} and Lemma~\ref{l:concentration-sub-gaussian}. Recall that $\bar s = \tr \bar \P_\perp \Sigmab \le \tr \Sigmab = r$ and $k < r$, we have
\begin{equation}\label{eq:bound-T11}
  T_{1,1} \le \sqrt{ (\bar s - s)^2 + C \big(|\bar s - s| + 2 r\big) }.
\end{equation}

It remains to bound $|\bar s - s|$. Note that $\P=(\X^\top\X)^\dagger\X^\top\X = \X^\top\X
(\X^\top\X)^\dagger$ and is symmetric, so
\begin{align*}
  &\I - \E[\P_\perp]\bar\P_\perp^{-1} + \I - \bar\P_\perp^{-1} \E[\P_\perp]
  = 2\E[\P] - \E[\gamma\P_\perp\Sigmab] - \E[\gamma\Sigmab\P_\perp]\\
  &= \sum_{i=1}^k\E\big[(\X^\top\X)^\dagger\x_i\x_i^\top + \x_i\x_i^\top (\X^\top\X)^\dagger \big] - \gamma (\E[\P_\perp]\Sigmab + \Sigmab \E[\P_\perp])\\
  &=\gamma\,
    \E\bigg[\bar s\cdot\frac{(\I-\P_{-k})\x_k\x_k^\top + \x_k\x_k^\top (\I-\P_{-k})} {\x_k^\top(\I - \P_{-k}) \x_k}\bigg]
    -\gamma\,\E\bigg[\hat s
    \cdot\frac{(\I-\P_{-k})\x_k\x_k^\top + \x_k\x_k^\top (\I-\P_{-k})}{\x_k^\top(\I-\P_{-k})\x_k}\bigg] \\ 
  & \qquad+\gamma\,\Big(\E\big[(\I - \P_{-k})\Sigmab\big] + \E\big[\Sigmab ( \I -\P_{-k}) \big]\Big)- \gamma \big(\E[\P_\perp]\Sigmab + \Sigmab \E[\P_\perp]\big) \\
  &=\gamma\,\E\bigg[(\bar s - \hat s)
    \cdot\frac{(\I-\P_{-k})\x_k\x_k^\top + \x_k\x_k^\top (\I-\P_{-k})}{\x_k^\top(\I-\P_{-k})\x_k}\bigg]
+    \gamma (\E[\P-\P_{-k}]\Sigmab + \Sigmab \E[\P-\P_{-k}]). 
\end{align*}
Moreover, using the fact that $\bar\P_\perp\Sigmab\preceq\frac1{\gamma+1}\I$ and $\bar \P_\perp \Sigmab = \Sigmab \bar \P_\perp$, we obtain that
\begin{align*}
  &|\bar s - s|  = |\tr (\bar \P_\perp - \E[\I - \P_{-k}] )\Sigmab|   \le |\tr (\bar \P_\perp - \E[\P_\perp] )\Sigmab|   + |\tr \E[\P - \P_{-k}] \Sigmab| \\
  &= \frac12 \big|\tr\,(\I-\E[\P_\perp]\bar\P_\perp^{-1})\bar\P_\perp\Sigmab + \tr\,\bar\P_\perp (\I - \bar \P_\perp^{-1} \E[\P_\perp]) \Sigmab \big|   + \tr\,\E \left[ \frac{(\I - \P_{-k})\x_k\x_k^\top(\I - \P_{-k})}{\x_k^\top (\I - \P_{-k}) \x_k} \right] \Sigmab \\
  &\leq \frac12 \big|\tr\,(\I-\E[\P_\perp]\bar\P_\perp^{-1} + \I - \bar \P_\perp^{-1} \E[\P_\perp])\bar\P_\perp\Sigmab \big|   + 1 \\
  &\leq \frac{\gamma}2 \,\E\bigg[|\bar s-\hat s|\cdot
    \frac{\tr\,( (\I -\P_{-k})\x_k\x_k^\top + \x_k\x_k^\top (\I -\P_{-k})) \bar\P_\perp\Sigmab}{\tr\,(\I -\P_{-k})\x_k\x_k^\top}  \bigg] \\
  &+ \gamma\, \E\bigg[\frac{\tr\,(\I-\P_{-k})\x_k\x_k^\top(\I-\P_{-k})\bar\P_\perp\Sigmab}
    {\tr\,(\I-\P_{-k})\x_k\x_k^\top}\bigg] + 1\\
  &\leq \frac{\gamma}{\gamma+1}\Bigg(\E\bigg[|\bar s-\hat
    s| \cdot \frac{\x_k^\top (\I -\P_{-k}) \x_k}{\x_k^\top (\I -\P_{-k}) \x_k}  \bigg] + 1\Bigg) + 1\leq \frac{\gamma}{\gamma+1}\Big(|\bar
    s-s| + \E\big[|s-\hat s|\big] + 1\Big) + 1 \\
  &\leq \frac{\gamma}{\gamma+1}\Big(|\bar s- s| +
     C\sqrt{|\bar s- s|} + C\sqrt{2 r} + 1\Big) + 1.
\end{align*}

Solving for $|\bar s-s|$, we deduce that
\begin{align*}
  |\bar s-s| \leq C_1 \sqrt{r} + C_2,
\end{align*}
so plugging back to \eqref{eq:bound-T11} we get $T_{1,1} \le C
\sqrt r$ and $\| \T_1\| \le \frac{C_\rho}{\sqrt r} $, thus completing
the proof.
\section{Convergence analysis of randomized iterative methods}
\label{a:newton}

Here, we discuss how our surrogate expressions for the expected
residual projection can be used to perform convergence analysis for
several randomized iterative optimization methods discussed in
Section~\ref{s:newton}.

\subsection{Generalized Kaczmarz method}
Generalized Kaczmarz \cite{generalized-kaczmarz} is an iterative
method for solving an $m\times n$ linear system
$\A\x=\b$, which uses a $k\times m$ sketching matrix $\S_t$ to reduce
the linear system and update an iterate $\x^t$ as follows:
  \begin{align*}
    \x^{t+1} = \argmin_\x\|\x-\x^t\|^2\quad\textnormal{subject to}\quad\S_t\A\x=\S_t\b.
  \end{align*}
Assume that $\x^*$ is the unique solution to the linear system $\A\x=\b$. In Theorems 4.1 and 4.6, \cite{generalized-kaczmarz} show that the
expected trajectory of the generalized Kaczmarz iterates, as
they converge to $\x^*$, is controlled
by the projection matrix $\P=(\S_t\A)^\dagger\S_t\A$ as follows:
\begin{align*}
  \text{(\cite{generalized-kaczmarz}, Theorem 4.1)}\qquad\quad \ \ \,
  \E[\x^{t+1}-\x^*]
  &= \big(\I-\E[\P]\big)\,\E[\x^t-\x^*],\\
  \text{(\cite{generalized-kaczmarz}, Theorem 4.6)}\qquad
  \E\big[\|\x^{t+1}-\x^*\|^2\big]
&\leq (1-\kappa)\,\E\big[\|\x^t-\x^*\|^2\big],
\text{ where }\kappa=\lambda_{\min}\big(\E[\P]\big).
\end{align*}
Both of these results depend on the expected projection
$\E[\P]$. The first one describes the expected trajectory of the
iterate, whereas the second one gives the worst-case convergence
rate in terms of the so-called \emph{stochastic condition number}
$\kappa$. We next demonstrate how Theorem \ref{t:main} can be used  in 
combination with the above results to
obtain convergence analysis for generalized Kaczmarz which is
formulated in terms of the spectral properties of $\A$. This includes
precise expressions for both the expected trajectory and $\kappa$. The following
result is a more detailed version of Corollary \ref{c:kaczmarz} from
Section \ref{s:newton}.
\begin{corollary}\label{c:kaczmarz2}
Let $\sigma_i$ denote the singular values of $\A$, and let $k$ denote
the size of sketch $\S_t$. Define:
\begin{align*}
\Delta_t=\x^t-\x^*\quad\text{and}\quad  \bar\Delta_{t+1} =
  (\gamma\A^\top\A+\I)^{-1}\E[\Delta_t]\quad\text{s.t.}\quad
  \sum_i\frac{\gamma\sigma_i^2}{\gamma\sigma_i^2+1}=k.
\end{align*}
  Suppose that $\S_t$ has i.i.d.~mean-zero sub-gaussian entries and let
  $r=\|\A\|_F^2/\|\A\|^2$ be the stable rank of $\A$. 
Assume that $\rho = r/k$ is a constant larger than $1$. Then, the
expected trajectory satisfies:
\begin{align}
  \big\|\E[\Delta_{t+1}] - \bar\Delta_{t+1}\big\|
  \leq \epsilon\cdot \|\bar\Delta_{t+1}\|,
  \quad\text{for}\quad\epsilon=O\big(\tfrac1{\sqrt
  r}\big).\label{eq:trajectory}
\end{align}
Moreover, we obtain the following worst-case convergence guarantee:
\begin{align}
  \E\big[\|\Delta_{t+1}\|^2\big]
  \leq \big(1-
  (\bar\kappa-\epsilon)\big)\,\E\big[\|\Delta_t\|^2\big],
  \quad\text{where}\quad
  \bar\kappa = \frac{\sigma_{\min}^2}{\sigma_{\min}^2+1/\gamma}.\label{eq:worst-case}
\end{align}
\end{corollary}
\begin{remark}
Our worst-case convergence guarantee \eqref{eq:worst-case}
  requires the matrix $\A$ to be sufficiently well-conditioned so
  that $\bar\kappa-\epsilon>0$. However, we believe that our surrogate expression
  $\bar\kappa$ for the stochastic condition number is far more
  accurate than suggested by the current analysis. 
\end{remark}

\subsection{Randomized Subspace Newton}
Randomized Subspace Newton (RSN, \cite{Gower2019})  is a randomized
Newton-type method  for minimizing a
smooth, convex and twice differentiable function
$f:\R^d\times \R$. The iterative 
update for this algorithm is defined as follows:
\begin{align*}
  \x^{t+1} = \x^t - \frac1L\S_t^\top(\S_t\H(\x^t)\S_t^\top)^\dagger\S_t\g(\x^t),
\end{align*}
where $\H(\x^t)$ and $\g(\x^t)$ are the Hessian and gradient of $f$ at
$\x^t$, respectively, whereas $\S_t$ is a $k\times d$ sketching
matrix (with $k\ll d$) which is refreshed at every iteration. Here,
$L$ denotes the \emph{relative smoothness} constant defined by
\cite{Gower2019} in Assumption 1, which also defines relative strong
convexity, denoted by $\mu$. In Theorem 2, they prove the
following convergence guarantee for RSN:
\begin{align*}
  \E[f(\x^t)] - f(\x^*) \leq \Big(1-\kappa\frac{\mu}{L}\Big)^t(f(\x^0)-f(\x^*)),
\end{align*}
where $\kappa =\min_{\x}\kappa(\x)$ and
$\kappa(\x)=\lambda_{\min}^+(\E[\P(\x)])$ is the smallest positive
eigenvalue of the expectation of the projection matrix
$\P(\x)=
\H^{\frac12}(\x)\S_t^\top(\S_t\H(\x)\S_t^\top)^\dagger\S_t\H^{\frac12}(\x)$. Our
results lead to the following surrogate expression for this expected
projection when the sketch is sub-gaussian:
\begin{align*}
  \E[\P(\x)] \simeq \H(\x)\big(\H(\x) +
  \tfrac1{\gamma(\x)}\I\big)^{-1}\quad\text{for}\quad\gamma(\x)>0
  \quad\text{s.t.}\quad\tr\,\H(\x)\big(\H(\x) +
  \tfrac1{\gamma(\x)}\I\big)^{-1} = k.
\end{align*}
Thus, the condition number $\kappa$ of RSN can be estimated using the
following surrogate expression:
\begin{align*}
\kappa\simeq \bar\kappa :=
  \min_\x\frac{\lambda_{\min}^+(\H(\x))}{\lambda_{\min}^+(\H(\x)) +
  1/\gamma(\x)}. 
\end{align*}
Just as in Corollary \ref{c:kaczmarz2}, an approximation
of the form $|\bar\kappa-\kappa|\leq\epsilon$ can be shown
from Theorem \ref{t:main}.
\begin{corollary}\label{c:rsn}
  Suppose that sketch $\S_t$ has size $k$ and i.i.d.~mean-zero sub-gaussian entries. Let
  $r=\min_\x\tr\,\H(\x)/\|\H(\x)\|$ be the (minimum) stable rank of
  the (square root) Hessian and assume that $\rho = r/k$ is a constant larger than
  $1$. Then,
  \begin{align*}
    |\kappa-\bar\kappa| \leq O\big(\tfrac1{\sqrt r}\big).
  \end{align*}
\end{corollary}

\subsection{Jacobian Sketching}
Jacobian Sketching (JacSketch, \cite{jacsketch}) defines an $n\times n$ positive semi-definite weight matrix
$\W$, and combines it with an $k\times n$ sketching matrix $\S$
(which is refreshed at every iteration of the algorithm), to
implicitly construct the following projection matrix:
\begin{align*}
  \Pi_\S = \S^\top(\S\W\S^\top)^\dagger\S\W,
\end{align*}
which is used to sketch the Jacobian at the current iterate (for the
complete method, we refer to their Algorithm 1).
The convergence rate guarantee given in their Theorem 3.6 for
JacSketch is given in terms of the Lyapunov function:
\begin{align*}
 \Psi^t = \|\x^t-\x^*\|^2 + \frac{\alpha}{2\mathcal{L}_2}\|\mathbf{J}^t-\nabla F(\x^*)\|_{\W^{-1}}^2,
\end{align*}
where $\alpha$ is the step size used by the algorithm. Under
appropriate choice of the step-size, Theorem~3.6 states that:
\begin{align*}
\E[\Psi^t] \leq \bigg(1 - \mu\,\min\Big\{\frac{1}{4\mathcal{L}_1},\frac{\kappa}{4\mathcal{L}_2\rho/n^2+\mu}\Big\}\bigg)^t\cdot\Psi^0,
\end{align*}
where $\kappa=\lambda_{\min}(\E[\Pi_\S])$ is the 
\emph{stochastic condition number} analogous to the one defined for
the Generalized Kaczmarz method, $n$ is the data size and
parameters $\rho$, $\mathcal{L}_1$, $\mathcal{L}_2$ and $\mu$ are
problem dependent constants defined in Theorem~3.6. Similarly as
before, we can use our surrogate expressions 
for the expected residual projection to obtain a precise estimate for
the stochastic condition number 
$\kappa$ under sub-gaussian sketching:
\begin{align*}
  \kappa\simeq\bar\kappa :=\frac{\lambda_{\min}(\W)}{\lambda_{\min}(\W) +
  1/\gamma}
  \quad\text{for}\quad\gamma>0\quad\text{s.t.}\quad\tr\,\W(\W +
  \tfrac1\gamma\I)^{-1} = k.
\end{align*}
\begin{corollary}\label{c:jacsketch}
  Suppose $\S_t$ has size $k$ and i.i.d.~mean-zero sub-gaussian entries. Let
  $r=\tr\,\W/\|\W\|$ be the stable rank of $\W^{\frac12}$ and assume
  that $\rho = r/k$ is a constant larger than 
  $1$. Then,
  \begin{align*}
    |\kappa-\bar\kappa| \leq O\big(\tfrac1{\sqrt r}\big).
  \end{align*}
\end{corollary}

\subsection{Omitted proofs}
\begin{proofof}{Corollary}{\ref{c:kaczmarz2}}
Using Theorem \ref{t:main}, for $\bar\P_{\perp}$ as defined in \eqref{eq:surrogate}, we have
  \begin{align*}
    (1-\epsilon)\bar\P_{\perp}\preceq \I-\E[\P] =
    \E[\P_{\perp}]\preceq(1+\epsilon)\bar\P_{\perp},
    \quad\text{where}\quad\epsilon=O\big(\tfrac 1{\sqrt r}\big).
  \end{align*}
In particular, this implies that
$\|\bar\P_{\perp}^{-\frac12}(\E[\P_{\perp}]-\bar\P_{\perp})\bar\P_{\perp}^{-\frac12}\|\leq\epsilon$.
Moreover, in the proof of Theorem \ref{t:main-tech} we showed that
  $\frac{\rho-1}\rho\I\preceq\bar\P_{\perp}\preceq\I$, see
  \eqref{eq:bound-bar-gamma-P}, so it follows that:
  \begin{align*}
    \bar\P_{\perp}^{-1}(\E[\P_{\perp}]-\bar\P_{\perp})^2\bar\P_{\perp}^{-1}\preceq
    \frac{\rho}{\rho-1}\big(\bar\P_{\perp}^{-\frac12}(\E[\P_{\perp}]-\bar\P_{\perp})\bar\P_{\perp}^{-\frac12}\big)^2
    \preceq     \frac{\rho}{\rho-1}\,\epsilon^2\cdot \I,
  \end{align*}
  where note that  $\frac{\rho}{\rho-1}\,\epsilon^2=O(1/r)$, since
  $\rho$ is treated as a constant. Thus we conclude that:
  \begin{align*}
    \|\E[\Delta_{t+1}]-\bar\Delta_{t+1}\|^2
    &=\E[\Delta_t]^\top(\E[\P_{\perp}]-\bar\P_{\perp})^2\E[\Delta_t]
    \\
    &\leq O(1/r)\cdot \E[\Delta_t]^\top\bar\P_{\perp}^2\E[\Delta_t] =
      O(1/r)\cdot \|\bar\Delta_{t+1}\|^2,
  \end{align*}
  which completes the proof of \eqref{eq:trajectory}. To show
  \eqref{eq:worst-case}, it suffices to observe that
  \begin{align*}
    \lambda_{\min}(\E[\P])
    = 1- \lambda_{\max}(\E[\P_{\perp}])
    \geq 1-(1+\epsilon)\lambda_{\max}(\bar\P_{\perp})
    \geq \lambda_{\min}(\I-\bar\P_{\perp}) - \epsilon,
  \end{align*}
  which completes the proof since
  $\I-\bar\P_{\perp}=\gamma\A^\top\A(\gamma\A^\top\A+\I)^{-1}$. 
\end{proofof}

Corollaries \ref{c:rsn} and \ref{c:jacsketch} follow analogously from
Theorem \ref{t:main}.

\begin{figure}
  \centering
  \includegraphics[width=.47\textwidth]{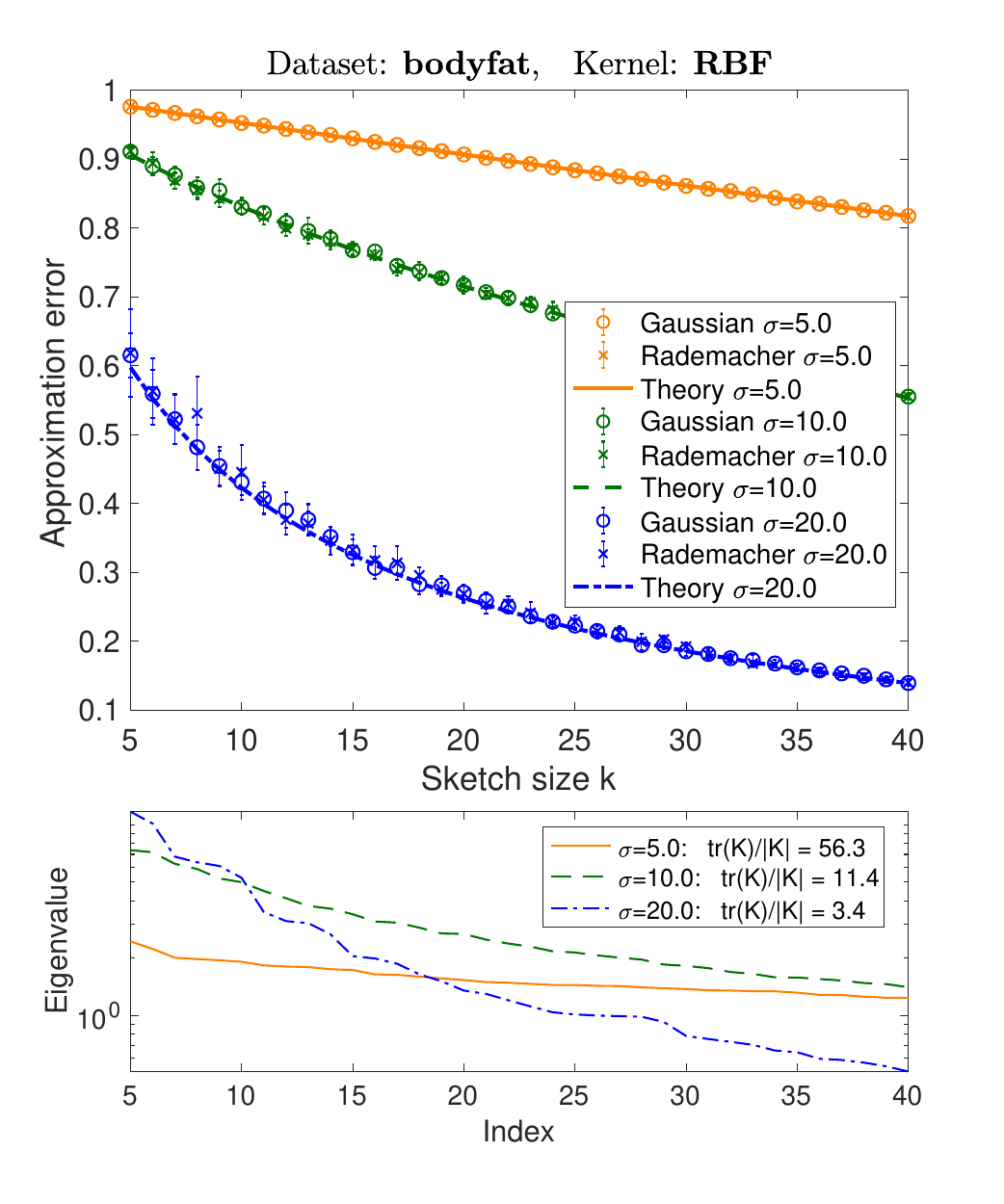}\nobreak\includegraphics[width=.47\textwidth]{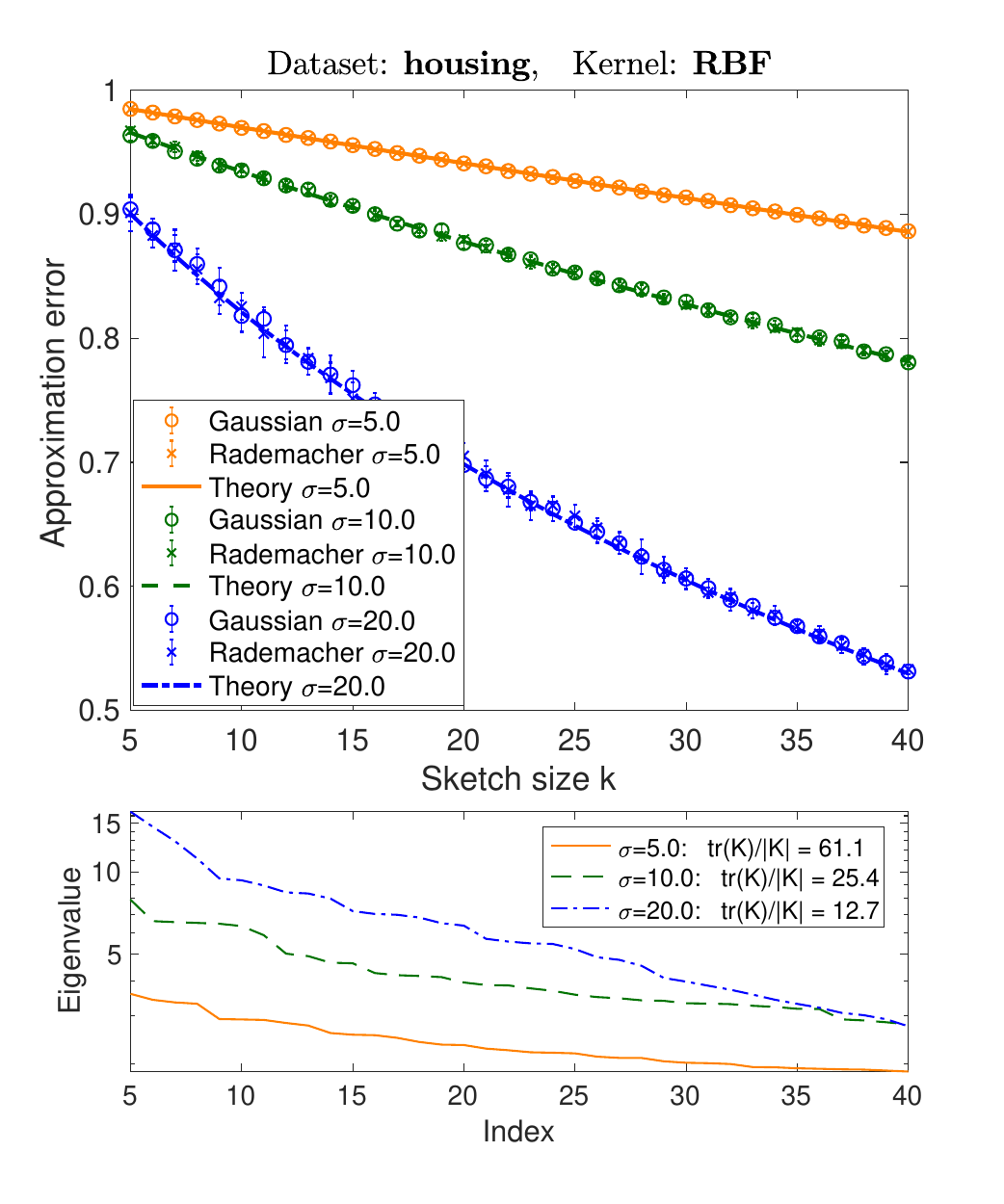}
  \includegraphics[width=.47\textwidth]{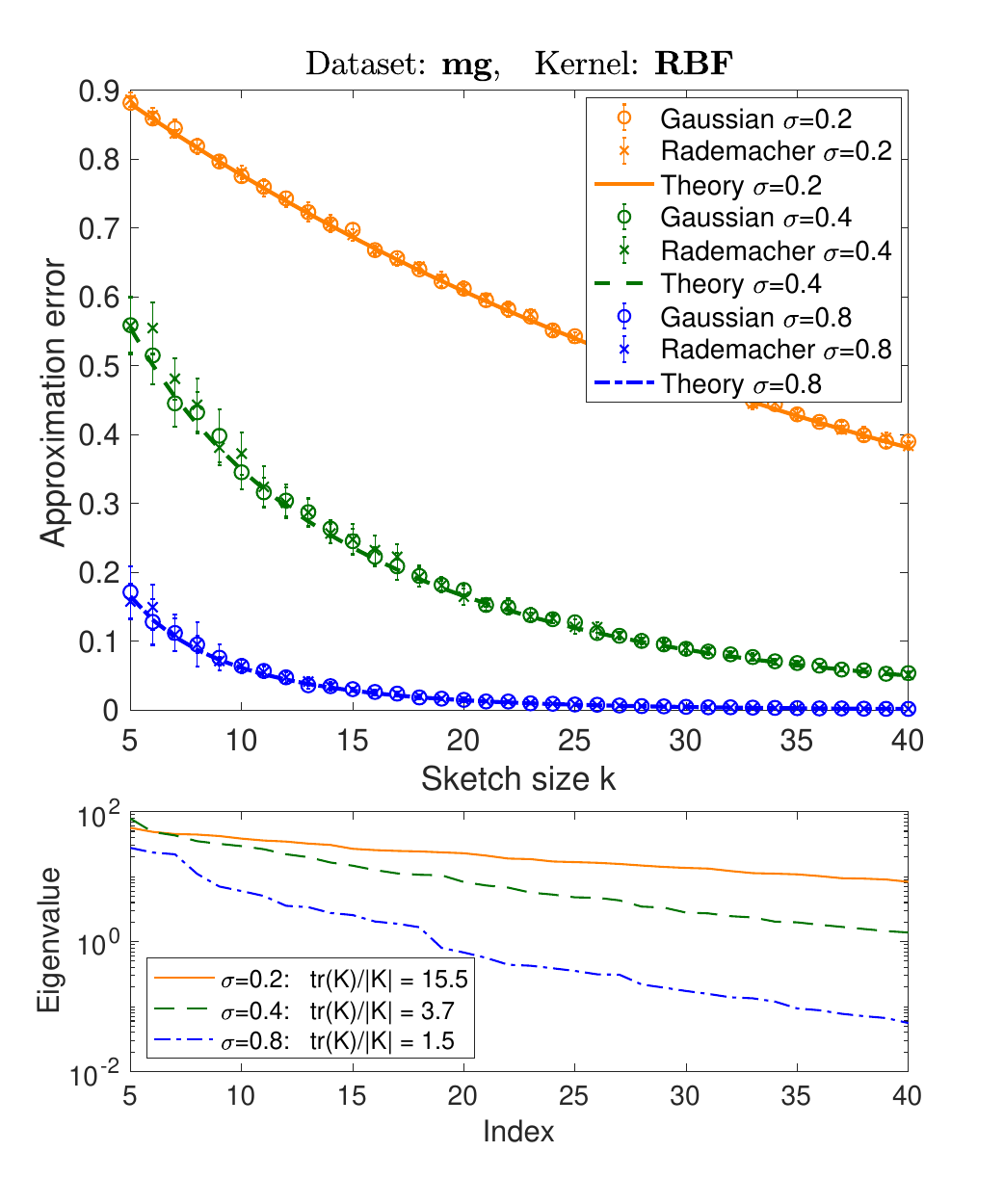}\nobreak\includegraphics[width=.47\textwidth]{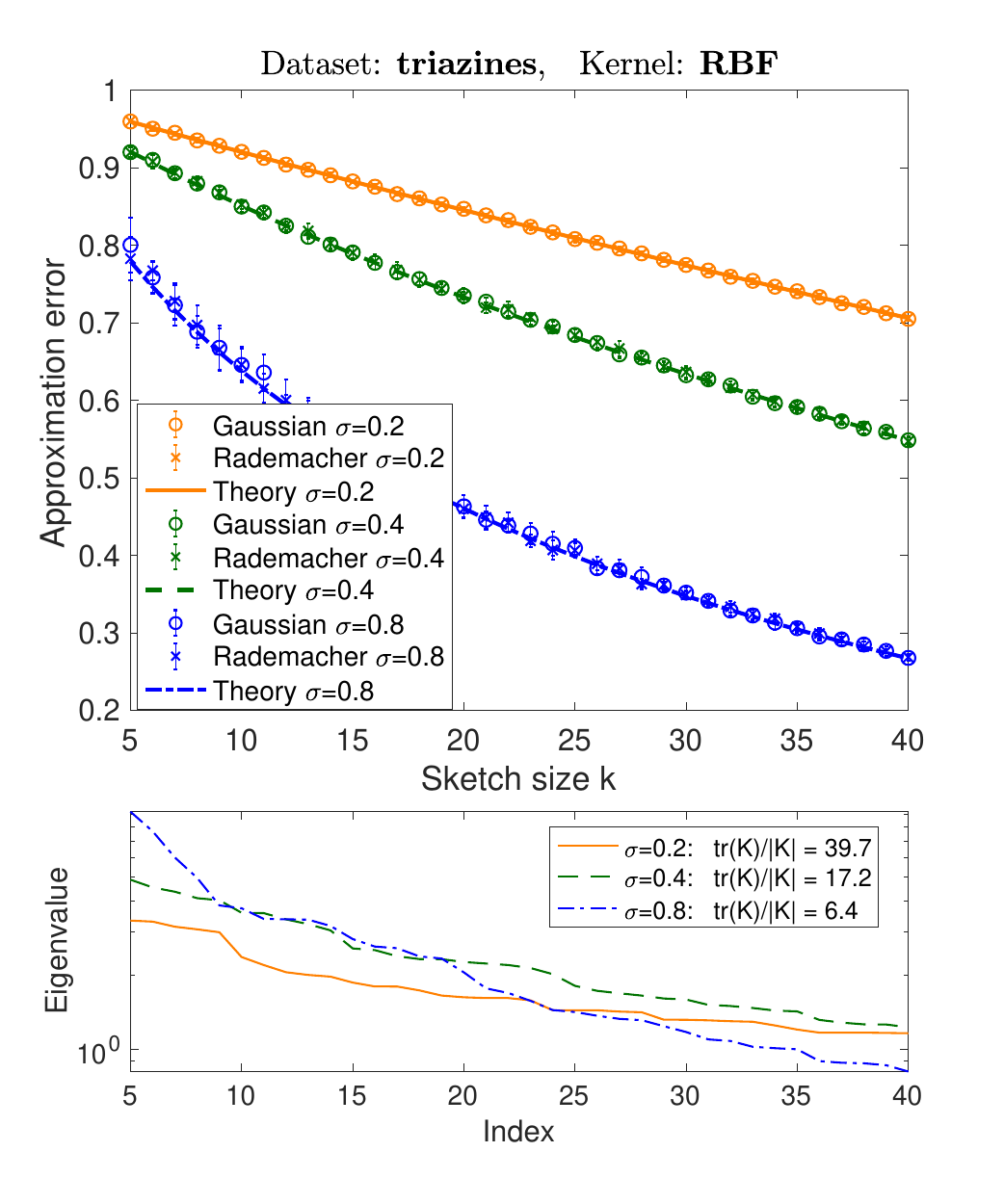} 
  \caption{Theoretical predictions versus approximation error for the
    sketched Nystr\"om with the RBF kernel, using Gaussian and
    Rademacher sketches (spectral decay shown at the bottom).}\label{f:nystrom2}
\end{figure}

\section{Additional empirical results}
\label{a:experiments}

We complement the results of Section \ref{s:experiments} with
empirical results on four additional libsvm datasets \cite{libsvm} (bringing
the total number of benchmark datasets to eight), which further
establish the accuracy of our surrogate expressions for the low-rank approximation
error. Similary as in Figure \ref{f:nystrom}, we use the sketched
Nystr\"om method \cite{revisiting-nystrom} with the RBF kernel
$k(\a_i,\a_j)=\exp(-\|\a_i-\a_j\|^2/(2\sigma^2))$, for several values
of the parameter $\sigma$. The values of $\sigma$ were chosen so as to
demonstrate the effectiveness of our theoretical predictions both when the
stable rank is moderately large and when it is very small.

In Figure \ref{f:nystrom2} we show the results for both Gaussian and
Rademacher sketches. These results reinforce the conclusions we made in
Section \ref{s:experiments}: our theoretical estimates are very
accurate in all cases, for both sketching methods, and even when the
stable rank is close to 1 (a regime that is not supported by the
current theory).

\end{document}